\documentclass{article}

     \usepackage[preprint]{neurips_2019}

\usepackage[utf8]{inputenc} 
\usepackage[T1]{fontenc}    
\usepackage{hyperref}       
\usepackage{url}            
\usepackage{booktabs}       
\usepackage{amsfonts}       
\usepackage{nicefrac}       
\usepackage{microtype}      
\usepackage{chngcntr}

\PassOptionsToPackage{hyphens}{url}
\usepackage{graphicx}
\usepackage{grffile}
\usepackage{longtable}
\usepackage{wrapfig}
\usepackage{rotating}
\usepackage[normalem]{ulem}
\usepackage{amsmath}
\usepackage{amsthm}
\usepackage{textcomp}
\usepackage{amssymb}
\usepackage{capt-of}
\usepackage[english]{babel}
\newtheorem{lemma}{Lemma}
\usepackage[boxed]{algorithm2e}

\DeclareMathOperator*{\argmax}{arg\,max}
\DeclareMathOperator*{\argmin}{arg\,min}
\newtheorem{theorem}{Theorem}[section]
\usepackage[x11names]{xcolor}
\hypersetup{linktoc = all, colorlinks = true, urlcolor = DodgerBlue4, citecolor = DodgerBlue4, linkcolor = black}

\title{Designing over uncertain outcomes with stochastic sampling Bayesian optimization}

\author{%
  Peter D.~Tonner\\
  Statistical Engineering Division\\
  National Institute of Standards and Technology\\
  Gaithersburg, MD 20874\\
  \texttt{peter.tonner@nist.gov} \\
  \And
  Daniel V. Samarov\\
  Statistical Engineering Division\\
  National Institute of Standards and Technology\\
  Gaithersburg, MD 20874\\
  \texttt{daniel.samarov@nist.gov} \\
  \And
  A. Gilad~Kusne\\
  Materials Measurement Science Division\\
  National Institute of Standards and Technology\\
  Gaithersburg, MD 20874\\
  \texttt{aaron.kusne@nist.gov} \\
}

\begin{document}

\maketitle

\begin{abstract}
Optimization is becoming increasingly common in scientific and engineering domains.
Oftentimes, these problems involve various levels of stochasticity or uncertainty in generating proposed solutions.
Therefore, optimization in these scenarios must consider this stochasticity to properly guide the design of future experiments.
Here, we adapt Bayesian optimization to handle uncertain outcomes, proposing a new framework called stochastic sampling Bayesian optimization (SSBO).
We show that the bounds on expected regret for an upper confidence bound search in SSBO resemble those of earlier Bayesian optimization approaches, with added penalties due to the stochastic generation of inputs.
Additionally, we adapt existing batch optimization techniques to properly limit the myopic decision making that can arise when selecting multiple instances before feedback.
Finally, we show that SSBO techniques properly optimize a set of standard optimization problems as well as an applied problem inspired by bioengineering.
\end{abstract}

\section{Introduction}
\label{sec:org8c00dd4}

Engineering tasks and scientific studies often rely on rapid
identification of an optimal prototype or experimental
condition. For instance, designing genetic sequences to improve
protein fitness \citep{romero-2012-navigating-protein-fitness}. Due
to the commonly high costs of generating proposed solutions at
each iteration and the complex shape of the objective being
targeted, interest has been growing around the use of Bayesian
optimization for these problems
\citep{shahriari-2016-taking-human-out-loop}. Bayesian optimization
(BO) combines statistical modeling with a quantitative
specification of an ideal search to rapidly identify best
solutions and have been applied to diverse industrial and
scientific endeavors including drug discovery
\citep{pyzer-knapp-2018-bayesian-optimization-accelerated},
aerospace engineering
\citep{hebbal-2019-multi-deep-gaussian-processes}, and alloy design
\citep{vellanki-2017-process-bayesian}.

Despite the widespread interest in applying BO to science and
engineering, there remain issues with standard BO techniques
limiting their impact. There are often real-world constraints
that violate the assumptions in standard BO, and various methods
have been developed to augment BO methods to address them
\citep{vellanki-2017-process-bayesian,azimi-2010-myopic,letham-2018-constrained-bayesian-optimization}.
In this work we aim to handle cases where candidate solutions are
not built exactly but instead are drawn from one of many sampling
distributions (Fig. \ref{fig:org46b0044}).  We refer to this approach as \emph{stochastic sampling}
BO (SSBO).

\begin{figure}[btp]
\centering
\includegraphics[width=\textwidth]{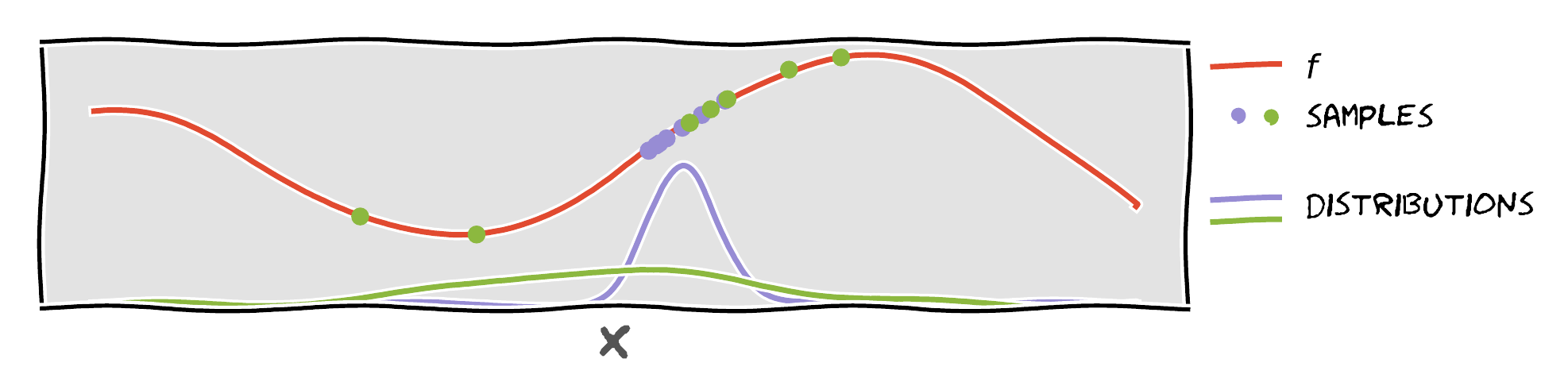}
\caption{\label{fig:org46b0044}
\textbf{Optimization under stochastic sampling} When sampling from an objective function \(f\) (red) stochastically, different sampling distributions must be considered (solid lines). The choice of distribution will impact both the maximum value of sampled observations as well as the variance of observed values.}
\end{figure}

Stochastic sampling occurs in many domains where iterative design
and discovery are made.  In synthetic biology, it is oftentimes
too costly to synthesize individual genetic variants compared to
generating large numbers of variants at once through a
randomization process called mutagenesis
\citep{kinney-2010-using-deep-sequencing,zheng-2017-targeted-mutagenesis}.
Additionally, synthetic biology is a field with a growing
interest in computer aided design
\citep{wu-2019-machine-learning-assisted}. In this work, we present
a simulated design of function through promoters, which control
the relative level of gene expression in bacteria and are a
common target of synthetic biology design studies
\citep{purnick-2009-second-wave-synthetic-biology}.  We focus on
this application in this study, but we believe stochastic
sampling is a common design constraint across engineering and
science.

In this paper, we propose a solution to SSBO that takes the
expectation of an upper confidence bound over each sampling
distribution. We show that our proposed algorithm achieves bounds
on regret comparable to standard BO techniques with an additional
constant term corresponding to sampling from a distribution at
each iteration. We then extend this approach to the situation
where multiple samples from each distribution is
desired. Finally, we test our SSBO procedure on synthetic
objective functions and a simulated bioengineering problem.

\section{Background}
\label{sec:org0af78e3}

\paragraph{Gaussian processes} Gaussian processes (GPs) are
non-parametric models of functional data, where any finite number
of function values are distributed as a multivariate normal
distribution
\citep{rasmussen-2006-gaussian-processes-machine-learning}. Specifically,
we obtain observations of an underlying process \(f(x)\) through
observations \(y(x)\), possibly with observation noise.  A GP is
defined by a mean function \(\mu(x)\) and covariance or kernel
function \(\kappa(x, x')\). The mean function is often assumed
fixed (\(\mu(x) = 0\)) and the behavior of the GP is governed by
the kernel. Kernels generally define an inner product between a
(possibly infinite) feature space on \(x\), and can be defined in a
number of ways depending on the context
\citep{hofmann-2008-kernel-methods-machine-learning}. In cases where
observation noise is present a variance term \(\sigma^2_y\) is
included. Typically, \(x \in \mathcal{R}^d\), although alternative
domains are possible when supported by the kernel.

When trained on data \(\mathbf{x}\), \(\mathbf{y}\), the predictive
mean and variance for a new observation \(y^*\) corresponding to
point \(x^*\) are
\begin{align}
   \mu(x^*) = \kappa(x^*, \mathbf{x}) (\kappa(\mathbf{x}, \mathbf{x}) + I \sigma^2_y)^{-1} \mathbf{y}
   \label{eq:gp-pred-mean}
\end{align}
and
\begin{align}
   \sigma^2(x^*) = \kappa(x^*, x^*) - \kappa(x^*, \mathbf{x}) (\kappa(\mathbf{x}, \mathbf{x}) + I \sigma^2_y)^{-1} \kappa(\mathbf{x}, x^*),
   \label{eq:gp-pred-variance}
\end{align}
respectively. Here, each \((x, y)\) pair corresponds to a
candidate solution \(x\) and the function observation \(y\).

\paragraph{Bayesian optimization} The purpose of BO is to identify
\begin{align}
 x^* = \argmax_{x \in D} f(x),
 \label{eq:optimization}
\end{align}
for some search space \(D\) and optimization target \(f\)
\citep{shahriari-2016-taking-human-out-loop,snoek-2012-practical-bayesian-optimization}.
BO is particularly well suited to searches over large spaces of
potential solutions, costly sampling procedures, and a target
function \(f(x)\) with many local optima. These techniques also
often carry rigorously defined bounds on the distance from the
global optimum at each iteration \(r_t = f(x^*) - f(x_t)\), called
the \emph{regret}
\citep{srinivas-2012-information-theoretic-regret}. 

BO combines a statistical model, typically a GP, with a
quantitative measure of the next desired observation, called the
acquisition
function. \citet{srinivas-2012-information-theoretic-regret}
established provable regret bounds when BO is conducted with a GP
as the function approximator and the upper confidence bound (UCB)
as the acquisition function. UCB has the form
\begin{align}
   \alpha_t(x) = \mu_t(x) + \beta^{1/2}_t \cdot \sigma_t(x),
   \label{eq:ucb}
\end{align}
where \(\mu_t\) and \(\sigma_t\) are the predictive mean and standard
deviation of the GP at iteration \(t\) and \(\beta_t\) is a
predefined, iteration-dependent value. Under mild assumptions of
the GP kernel and \(f\), selecting \(x_t\) from the maximum of the UCB leads
to sub-linear cumulative regret (\(R_T =
     \sum_{i=1}^T r_t\)) with high probability (e.g. \(lim_{t
     \rightarrow \infty} \frac{R_T}{t} = 0\)).

\paragraph{Batch Bayesian optimization} While standard BO assumes
that observations from the process \(f(x)\) are generated one at a
time (referred to as \emph{sequential} optimization), there has been
considerable effort to expand BO techniques to cases where more
than one observation is made at each iteration. These techniques
are referred to as \emph{batch} Bayesian optimization
\citep{kathuria-2016-batched-gaussian-process,gonzalez-2015-batch-bayesian-optimization,desautels-2014-parallelizing-exploration-exploitation}.
In order to avoid myopic over-exploitation, batch BO algorithms
approximate the feedback that would be received if selection was
performed in a sequential manner by modifying the acquisition
function during batch construction.

\paragraph{Constrained BO} Real-world applications often cannot
map directly to the standard BO framework. This has led to many
studies on the use of BO in the presence of constraints
\citep{azimi-2010-myopic,letham-2018-constrained-bayesian-optimization}.
Notable approaches are those which combine constraints on the
optimization problem with stochastic selection of new candidates
\citep{azimi-2016-budgeted-optimization-with,yang-2019-batched-stochastic-bayesian}. In
this work, we instead consider the expected reward given a
sampling distribution directly, and establish provable regret
bounds on BO performed in this manner.

\section{Optimization via stochastic samples}
\label{sec:org273a35c}
We consider the problem of maximizing a function \(f(x)\) when \(x\)
is sampled from a distribution \(\pi(x \vert \theta)\). Specifically
the goal is to solve Eq. \ref{eq:optimization} by choosing
sampling distribution parameters \(\theta \in \Theta\) that minimize
the expected regret with respect to the optimal \(f(x^*)\)
\begin{align}
  r_x(\theta) = E_{\pi(\theta)}[f(x^*) - f(x)] = f(x^*) -  E_{\pi(\theta)}[f(x)],
\end{align}
where \(E_{\pi(\theta)}\) is the expectation over the distribution
\(\pi(x \vert \theta)\).
Given that the choice of \(\theta\) impacts the regret only through
the expectation of \(f(x)\), we reframe the optimization in terms of
the optimal sampling distribution \(\theta^*\):
\begin{align}
   \theta^* = \argmin_{\theta \in \Theta} r_x(\theta) = \argmax_{\theta \in \Theta} E_{\pi(\theta)}[f(x)]. 
\end{align}
We also define the expected regret relative to the optimal
\(\theta^*\) for a chosen \(\theta\),
\begin{align}
   r_\pi(\theta) = & r_x(\theta) - r_x(\theta^*) \\
		    = & E_{\pi(\theta^*)}[f(x)] - E_{\pi(\theta)}[f(x)].
    \label{eq:rpi_theta}
\end{align}
Given that \(r_x(\theta^*)\) is fixed for a given \(f\) and \(\Theta\),
minimizing \(r_x\) (and \(r_\pi\)) corresponds to maximizing the
expectation of \(f\) over \(\theta\).     

Our goal is to develop an iterative procedure where at each
iteration \(t\), we select \(\theta_t\) to minimize our instantaneous
regret \(r_x(\theta)\). Ultimately, we aim to minimize the total
regret at \(T\) rounds \(R_T = \sum_{t = 1}^T r_x(\theta_t)\).  In
order to identify the \(\theta\) with minimal regret at iteration
\(t\), we adopt the UCB bound proposed in Eq. \ref{eq:ucb}.
Specifically, at each iteration \(t\), we select \(\theta_t\) with the
maximal expected value of \(\alpha_t(x)\):
\begin{align}
 \theta_t := \argmax_{\theta \in \Theta} E_{\pi(\theta)}[\alpha_t(x)].
 \label{eq:thetat}
\end{align}
The complete procedure, stochastic sampling GP-UCB (SS-GPUCB), can be seen in
Algorithm \ref{alg:sgpucb} and Fig. \ref{fig:org9b094d3}. In the next section, we describe
bounds on \(R_T\) for both discrete and continuous sampling
distributions when iterative values of \(\theta_t\) are selected using SS-GPUCB.

\begin{algorithm}[t]
   \label{alg:sgpucb}
   \KwData{Sampling distribution $\pi$, parameter space $\Theta$, GP prior with $\mu_0=0$, $\sigma^2_y$, $k$}
   \For{$t = 1, 2, \dots$}
   {
	   Define $\alpha_t(x) = \mu_{t-1}(x) + \beta_t^{1/2} \sigma_{t-1}(x)$\;
	   Choose $\theta_t = \argmax_{\theta \in \Theta} E_{\pi(\theta)}[\alpha_t(x)]$\;
	   Sample $x_t \sim \pi(\theta)$ and $y_t = f(x_t) + \epsilon$\;
	   Update GP with data point $(x_t, y_t)$\;
   }
   \caption{Stochastic sampling GPUCB algorithm}
\end{algorithm}

\begin{figure}[btp]
\centering
\includegraphics[width=\textwidth]{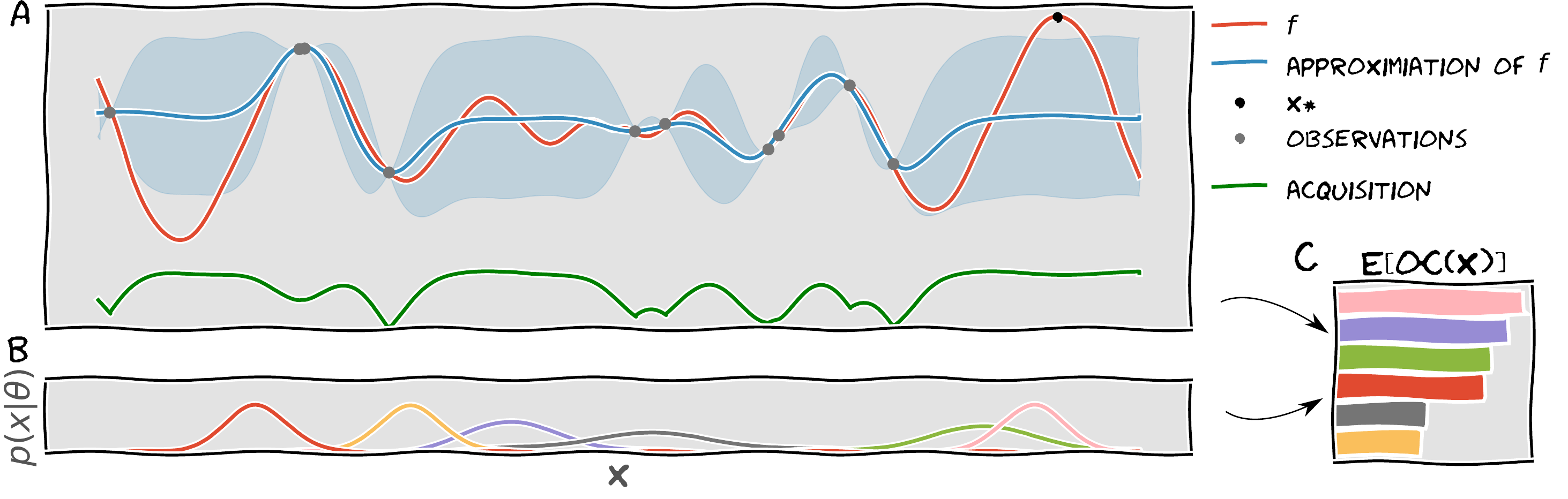}
\caption{\label{fig:org9b094d3}
\textbf{Stochastic sampling Bayesian optimization.} (A) Observations (gray) of the function \(f\) (red) are used to construct a GP (blue, solid line is mean and shaded region is 95\% confidence interval) and a standard BO acquisition function (\(\alpha(x)\), green). (B) Different sampling distributions (\(\pi(\theta)\)) are available from which to generate new \(x\) samples. This example uses normal distributions with parameters \(\theta = \{\mu, \sigma^2\}\). (C) The expectation of \(\alpha(x)\) for various values of \(\theta\).}
\end{figure}

\subsection{Bounding expected regret}
\label{sec:org55921d0}
The regret bounds established in this work build on those
constructed for standard GP-UCB
\citep{srinivas-2012-information-theoretic-regret}. The regret
bounds for SS-GPUCB include a term representing the maximal mutual
information between \(T\) observations \(Y_T\) and the true function
\(f\), \(\gamma_T = \max_{\vert A \vert = T, A \subset D} I(Y_A; f)\).
We define an additional constant in this work, relating the
regret bounds to the sampling distributions used for search.
For sampling distributions \(\pi(x \vert \theta)\), we define a
constant
\begin{align}
\pi^* := \max_{x \in D, \theta \in \Theta} \pi(x \vert \theta).
\end{align}
\(\pi^*\) corresponds to the maximal pdf value of \(\pi\) for all
possible \(x\) and \(\theta\) in the optimization problem.  In order
to ensure that \(\pi^*\) remains well defined, we assume that
\(\Theta\) is bounded. For example, in the case of a Gaussian
distribution where \(\theta = \{\mu, \sigma\}\), \(\lim_{\sigma
     \rightarrow 0} \pi^* = \infty\). In this case, we would assume
that \(\Theta\) will be defined such that \(\sigma > 0 \quad \forall
     \theta \in \Theta\). Using these terms, we now state the main
theoretical results of this work.

For both discrete and continuous distributions, the
bounds on the cumulative expected regret at iteration \(T\) are of the form 
\begin{align}
  \mathcal{O}^*(\sqrt{T\beta_T\gamma_T\pi^*d}).
  \label{eq:regret-general}
\end{align}
Here \(\mathcal{O}^*\) corresponds to a specialized form of the
standard \(\mathcal{O}\) where logarithmic terms are removed and
\(d\) is the size of the search space. Full proofs for
these bounds are available in the appendix. These bounds are
similar to that of
\citet{srinivas-2012-information-theoretic-regret}, with \(\sqrt{\pi^*d}\)
corresponding to an added impact of sampling from a distribution.
These bounds maintain the sub-linear cumulative regret of
standard GPUCB, enabling efficient optimization in stochastic
sampling scenarios.

\subsection{Optimizing over stochastic batch experiments}
\label{sec:org5b6faa1}
Experiments are often conducted generating multiple observations
for a given set of experiment parameters, i.e. a given
\(\theta\). In this case we wish to improve the selection of
\(\theta\) by considering the potential information shared between
individual observations. To this end, we adopt techniques from batch
BO
\citep{desautels-2014-parallelizing-exploration-exploitation}. Our
methods are similar to batch BO in that multiple observations
will be collected in each feedback iteration. Our approach
differs from these techniques, however, in the selection of a
single \(\theta\) value at each iteration from which many
observations will be drawn. We distinguish our approach from
other batch BO methods by referring to this as \emph{stochastic} batch
BO.

We adopt the technique of approximating the expected feedback
that would be received during sequential search through a penalty
applied on the acquisition function. This penalty, \(\varphi(x_i;
     x_j)\), defines the approximate impact that observing \(f(x_j)\)
would have on \(\alpha_t(x_i)\). It is a heuristic that acts as a
local penalizer around \(x_{j}\), meaning that it is
differentiable, \(0 \leq \varphi(x_i; x_j) \leq 1\), and
\(\varphi(x_i; x_j)\) is non-decreasing as the distance between
\(x_i\) and \(x_j\) grows (see appendix for explicit form)
\citep{gonzalez-2015-batch-bayesian-optimization}.

Using a penalty term enforces exploration when selecting \(\theta\)
by decreasing \(\alpha(x)\) around the positions most likely to be
sampled for a given \(\pi(\theta)\) (Fig. \ref{fig:org2477bcd}).  As
we show below, these methods enable independent marginalization
of approximate acquisition values. While other methods of
constructing batch samples exist for Bayesian optimization, they
require combinatorial searches over the previously sampled batch
values (see appendix). As such, we focus on heuristic approaches
here.

At iteration \(t\) we will sample \(B\) new observations from the
distribution \(\pi(x \vert \theta_t)\). For each new point \(x_i\)
(\(1 \leq i \leq B\)), given the previous observations in the batch
\(x_{1:i-1} = \{x_1, \dots, x_{i-1}\}\), the acquisition function
for \(x_i\) is
\begin{align}
   \alpha_t(x_i) \cdot \prod_{j=1}^{i-1} \varphi(x_i; x_{j}).
   \label{eq:aq-batch}
\end{align}

The advantage of the formulation of batch acquisition values in
Eq. \ref{eq:aq-batch} is that the acquisition can be easily
marginalized for each observation \(x_i\) over previous
observations \(x_{1:i-1}\)
(Fig. \ref{fig:org2477bcd}). Specifically, each element of
\(x_{1:i-1}\) is \emph{iid} and so the expectation of the penalty is
\begin{align}
   E_{\pi(\theta)} \Bigg[ \prod_{j=1}^{i-1} \varphi(x_i; x_j) \Bigg] = & \prod_{j=1}^{i-1} E_{\pi(\theta)} \Bigg[\varphi(x_i; x_j) \Bigg] \\
		   	 	  			  	     	   	 = & \prod_{j=1}^{i-1} \varphi_{\pi(\theta)}(x_i) = \big(\varphi_{\pi(\theta)}(x_i)\big)^{i-1}, 
										 \label{eq:expect-gamma}
\end{align}
where we introduce the function \(\varphi_{\pi(\theta)}(x_i)\)
representing the expected penalty over \(\pi(\theta)\)
(Fig. \ref{fig:org2477bcd}C). This then translates to
calculating the expected acquisition value for each iteration \(i\)
and ultimately for varying batch size \(B\)
(Fig. \ref{fig:org2477bcd}E,F).  We choose \(\theta_t\) such that
\begin{align}
  \theta_t = \argmax_{\theta \in \Theta} E_{\pi(\theta)} \Bigg[\sum_{k=1}^B \alpha_t(x) \varphi_{\pi(\theta)}^{k-1}(x) \Bigg].
  \label{eq:theta_t-batch}
\end{align}

This approach, which we call stochastic batch GPUCB (SB-GPUCB),
explicitly captures the trade-off between exploration and
exploitation for varying batch size and sample distribution
variance. Specifically, if batch size is small (e.g. approaching
\(B=1\) for sequential search) then sampling distributions with
smaller variance are preferred because they more precisely target
the current maximum (Fig. \ref{fig:org2477bcd}F, purple
bars). However, as batch size increases, broader sampling
distributions are preferred in order to increase the information
gained from multiple observations in a single batch
(Fig. \ref{fig:org2477bcd}F, yellow bars). This effect is attenuated
as the search begins to properly identify the function optimum,
however (Fig. \ref{fig:org151d5bf}). As the model becomes more
complete, and therefore the acquisition value becomes a more
accurate predictor of the function optimum, lower variance
sampling distributions are preferred regardless of batch
size. Therefore, the local penalty approach can properly adapt
the selection of \(\theta\) at each iteration to best take
advantage of the provided batch size and current knowledge of the
process \(f\).

\begin{figure}[tb]
\centering
\includegraphics[width=\textwidth]{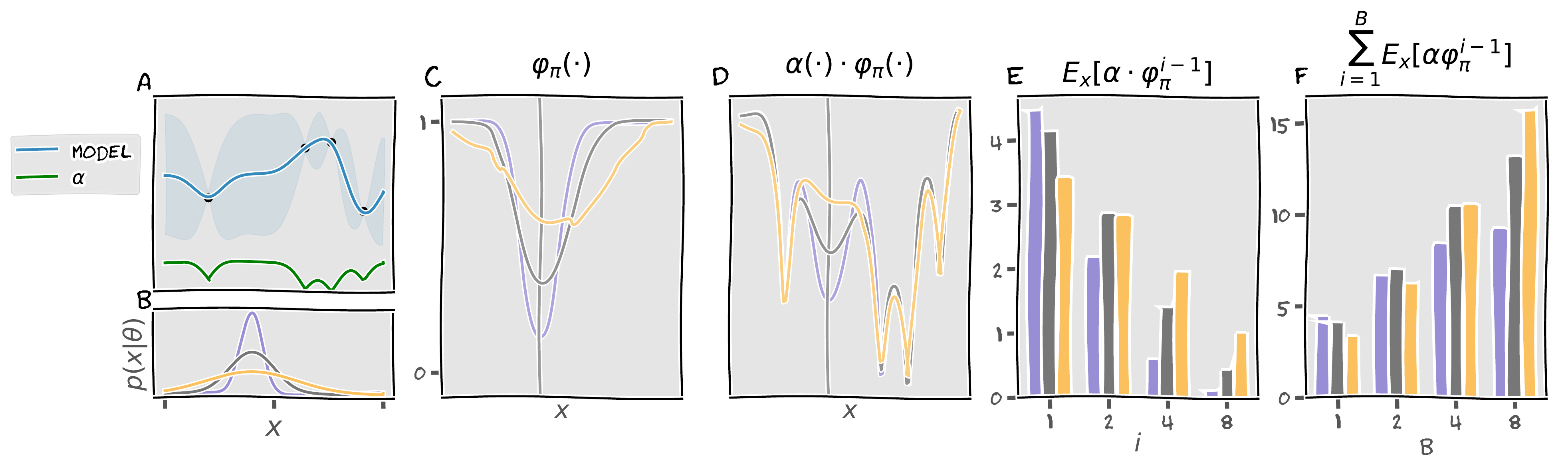}
\caption{\label{fig:org2477bcd}
\textbf{Local penalization of the acquisition function over stochastic batches}. We display the effect of a locally penalized aquisition function for a model and acquisition function (A). Three sampling distributions, with the same mean and differing variances will be compared (B). The expected penalty \(\varphi_\pi\) is different for each distribution, due to differences in probability mass over \(x\) (C). This impacts the expected \(\alpha\) differently for each distribution (D). Expectations for \(\alpha(x) \varphi_\pi^{i-1} (x)\) decreases with \(i\) but for a different rate for each distribution (E). This leads to different optimal distributions depending on batch size (F, Eq. \ref{eq:theta_t-batch}). \(x\) is removed from equations for simplicity.}
\end{figure}

\begin{algorithm}[t]
   \label{alg:sgpucb-batch}
   \KwData{Sampling distribution $\pi$, parameter space $\Theta$, batch size B, GP prior with $\mu_0=0$, $\sigma^2_y$, $k$}
   \For{$t = 1, B, 2B, \dots$}
   {
	   Define $\alpha_t(x) = \mu_{t-1}(x) + \beta_t^{1/2} \sigma_{t-1}(x)$\;
	   Define $\varphi_{\pi(\theta)} = E_{\pi(\theta)}[\varphi(\cdot; x)]$\;
	   Choose $\theta_t = \argmax_{\theta \in \Theta} E_{\pi(\theta)} \Big[\sum_{k=1}^B \alpha_t(x) \varphi_{\pi(\theta)}^{k-1}(x) \Big]$\;
	   Sample $x_j \sim \pi(\theta)$ and $y_j = f(x_j) + \epsilon$ for $t \leq j < t+B$\;
	   Update GP with data-points $(x_j, y_j)$, $t \leq j < t + B$\;
   }
   \caption{Batch stochastic sampling GPUCB algorithm}
\end{algorithm}

\section{Experiments}
\label{sec:org83984e2}

\paragraph{Objective functions} We selected objective functions
from the optimization literature for evaluating our SSBO
algorithms
\citep{surjanovic-virtual-library-simulation-experiments}. Details
of these functions can be found in the appendix.  They cover many
useful characteristics when comparing optimization algorithms
including many local optima, multiple periodicities and
magnitudes, and sharp ridge boundaries.

\paragraph{Alternative acquisition functions} We developed
alternative acquisition functions to compare against our own
procedure. We define the \emph{maximum mean} acquisition as
\(E_{\pi(\theta)}[\mu_t(x)]\), which takes the expectation over
\(\pi(\theta)\) of the predictive mean and is a previously
suggested exploitative strategy
\citep{azimi-2016-budgeted-optimization-with}. \emph{Mean GPUCB} is
defined as \(\alpha_t(E_{\pi(\theta)}[x])\) and corresponds to
considering only the mean of the distribution
\(\pi(\theta)\). Finally, an \emph{independent} model of \(\alpha_t\) is
used to test the impact of \(\varphi\) in batch sampling. The
\emph{independent} acquisition function is not relevant in sequential
search because \(\varphi\) is not used. We also compare to \emph{random}
search, where all \(\theta\) values are chosen uniformly at random.

\paragraph{Evaluating performance} At iteration \(t\), we consider
both instantaneous regret:
\begin{align}
   r_t = f(x^*) - f(x_t),
   \label{eq:instreg}
\end{align}
and simple regret:
\begin{align}
   \min_{1 \leq i \leq t} r_t.
   \label{eq:simpreg}
\end{align}
Each condition was run until a
total of \(200\) observations was received and results were
averaged over \(50\) simulations each.

\paragraph{Sampling distributions.} Sampling distributions were
constructed using a discretization of a normal
distribution. Means were placed at 32 evenly spaced positions
across both input dimensions, for a total of 1024 unique
two-dimensional mean positions. Five standard deviations were
chosen covering values from \(10^{-3}\) to \(2 \times 10^{-1}\) the
input dimension size.

\paragraph{Implementation} Simulations were run in Python, using
the GPy library for model inference \citep{gpy-2012-gpy}. The source
code for running the simulations is provided in the supplement.

\subsection{Sequential stochastic optimization}
\label{sec:org2e6b03b}
\begin{figure}[tbp]
\centering
\includegraphics[width=\textwidth]{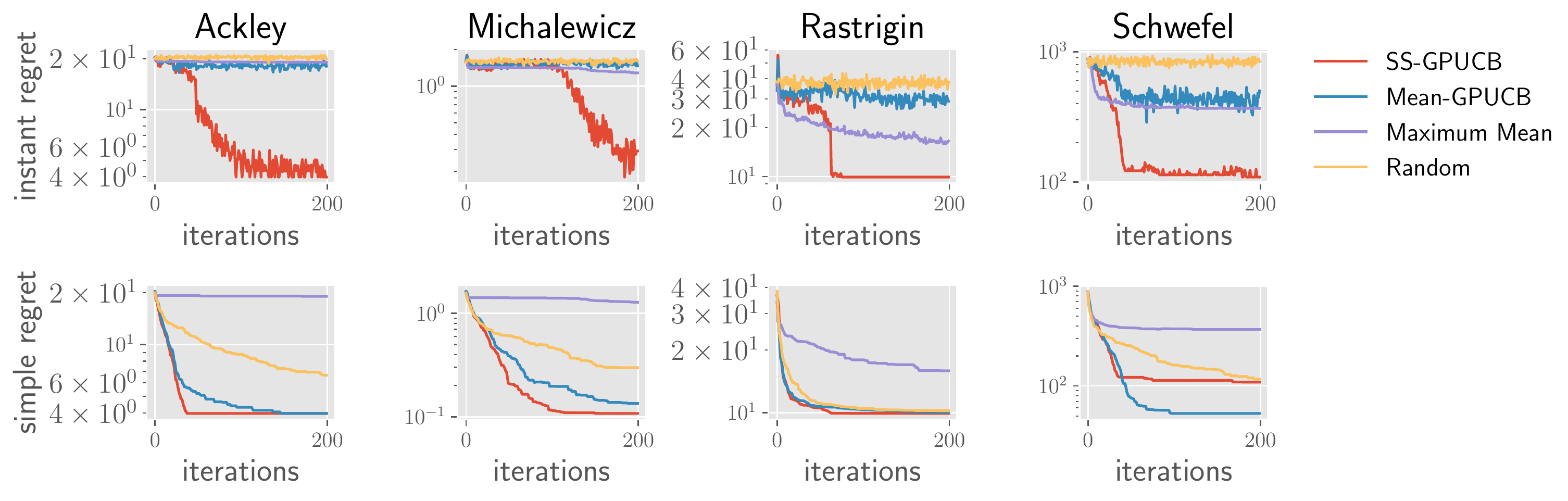}
\caption{\label{fig:orga3d32fd}
\textbf{Sequential SS-GPUCB regret.} For each acquisition function, we show the instantaneous regret (top, Eq. \ref{eq:instreg}) and simple regret (bottom, Eq. \ref{eq:simpreg}) achieved under each objective.}
\end{figure}

We first tested SS-GPUCB under sequential feedback and compared
it to other acquisition functions.  We found that under both
instantaneous and simple regret, SB-GPUCB outperformed all other
methods, with one exception (Fig. \ref{fig:orga3d32fd}). The one
exception occurred when comparing SS-GPUCB to the Mean-GPUCB
acquisition function under the Schwefel objective function. This
difference appeared to be due to an early convergence to the
local optima of that function in a small number of trials (Fig.
\ref{fig:org181cb08}). Of particular note is the stark
difference in performance of SS-GPUCB in instantaneous regret
compared to the other methods considered. The other methods do
not achieve considerable decrease in instantaneous regret over
the course of the simulation, indicating that they do not
properly combine the model predictions with the expected reward
induced by different sampling distributions. This also likely
indicates that gains in simple regret are due at least in part to
the added random chance of improvement created by sampling \(x\)
from a distribution \(\pi(\theta)\).

\subsection{Stochastic Batch optimization}
\label{sec:org909b539}
\begin{figure}[tb]
\centering
\includegraphics[width=\textwidth]{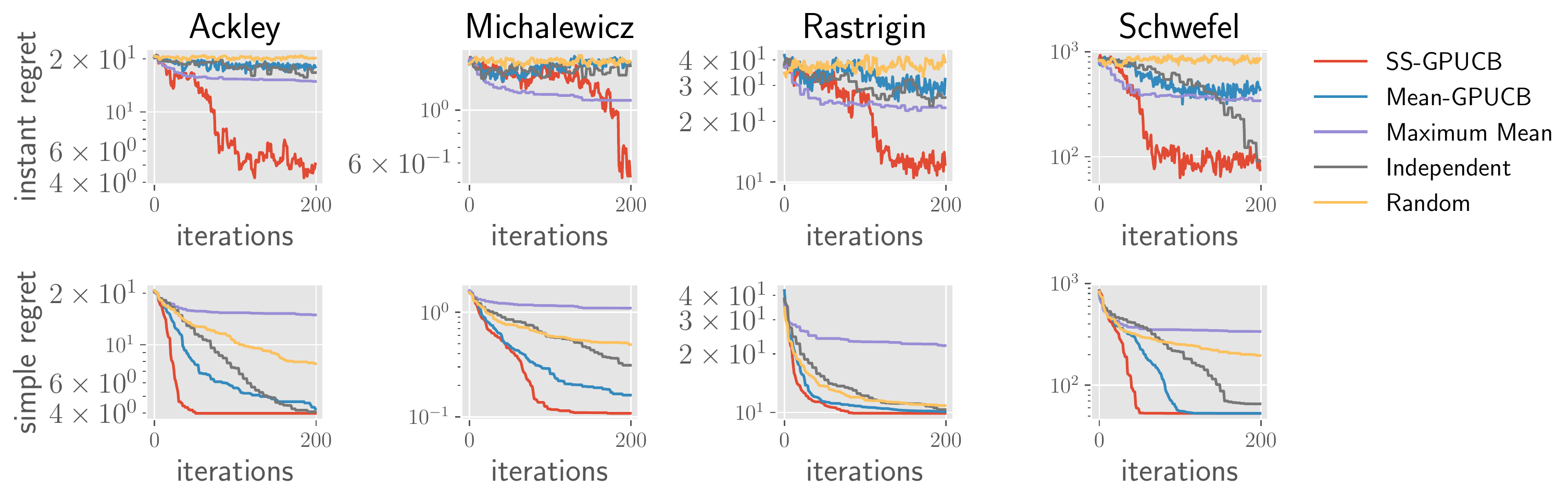}
\caption{\label{fig:org9715007}
\textbf{Batch SS-GPUCB regret} Instantaneous and simple regret for each objective function under batch optimization with a batch size of 5.}
\end{figure}

We next considered the ability of SB-GPUCB to optimize functions
under batch sampling. Again we found that compared to other
acquisition functions, SB-GPUCB rapidly identifies optimal values
of \(f\) (Fig. \ref{fig:org9715007}). This includes the comparison of a
SB-GPUCB acquisition with no local penalty for batch observations
(independent), which appears to lag considerably behind the
locally penalized acquisition in minimizing regret. This
indicates that approximating the change in \(\alpha\) from each
observation \(x_t\) using \(\varphi\) improves optimization over more
naive methods.

\begin{figure}[tb]
\centering
\includegraphics[width=\textwidth]{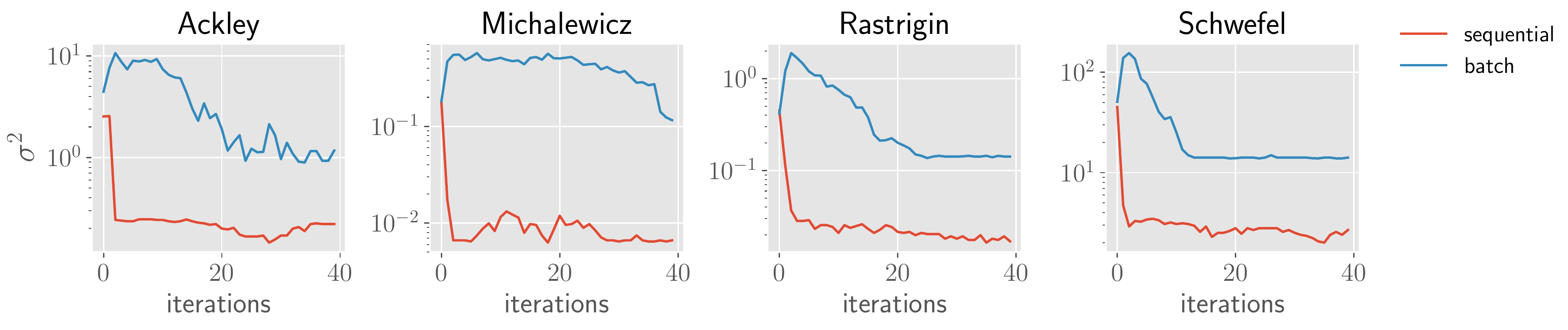}
\caption{\label{fig:orgc1be634}
\textbf{Variance selection under sequential and batch optimization.} Variance of the distribution selected by SS-GPUCB (sequential, red) and SB-GPUCB (batch, blue) at each iteration.}
\end{figure}

Of particular interest is how the sampling distribution \(\pi\) is
used in selecting the most optimal \(\theta\) at each
iteration. During sequential optimization, low variance
distributions will be most advantageous because they allow for
more precise selection of the next observation. However when
observations are collected in batches, it is more useful to
select high variance sampling distributions early in the search
to more rapidly explore the input space.  This behavior is
directly reflected in the sampling variances selected by SS-GPUCB
and SB-GPUCB under sequential and batch optimization,
respectively (Fig. \ref{fig:orgc1be634}). In particular SS-GPUCB
prefers low variance distributions throughout the simulation,
with the exception of the first iteration where no data is
available and all expectations of \(\alpha\) for different values
of \(\theta\) are considered equal. We also see that initially
SB-GPUCB selects higher variance early in the search and steadily
declines over time. However, even during later iterations the
variance selected during batch optimization does not converge to
that selected during sequential optimization.  This is due to the
fact that the batch size remains constant over the simulation and
an intermediate variance results in the highest expected return
over the combination of all batch observations. We expect that if
we adaptively selected the batch size, SB-GPUCB would choose the
minimum variance in combination with a decreasing batch size as
the model increasingly identifies the true optimum
\citep{desautels-2014-parallelizing-exploration-exploitation}.

\subsection{Stochastic batch optimization of biological function}
\label{sec:orgb20e960}
\begin{figure}[tb]
\centering
\includegraphics[width=\textwidth]{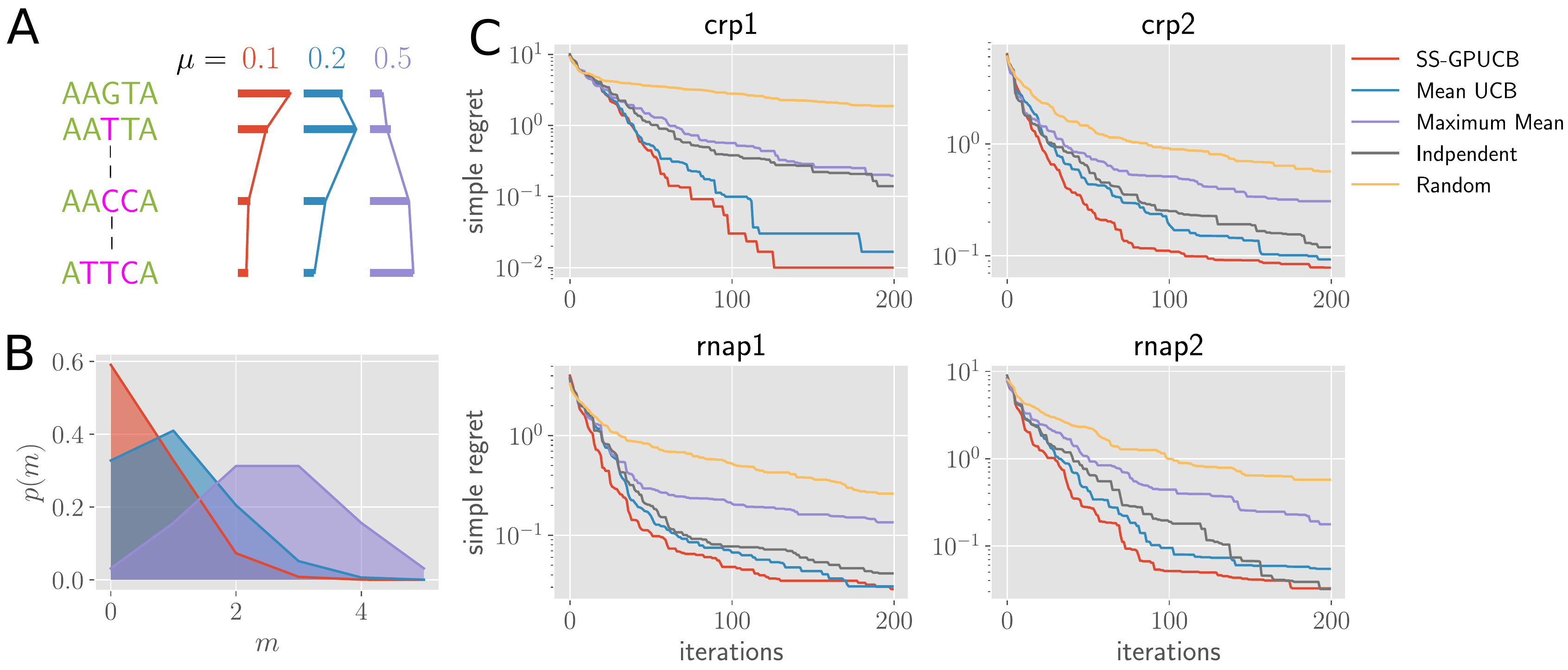}
\caption{\label{fig:org82b2f40}
\textbf{Stochastic batch Bayesian optimization of the \emph{lac} promoter.} (A) An initial DNA sequence (green) will be mutagenized with various mutation rates \(\mu\) (red, blue, and purple). The probability of each mutation (pink letters) depends on the corresponding mutation rate (red, blue, and purple bars). (B) Probability of mutation numbers \(m\) for different mutation rates \(\mu\) (same as in A). (C) Simple regret of each of the four regions of the \emph{lac} promoter for each acquisition function.}
\end{figure}

In order to validate our method for use in real world
applications, we evaluated the performance of SB-GPUCB on a
simulated problem in promoter design. Synthetic design of
promoters is an ideal candidate for SSBO because the search
space grows exponentially with the size of the genomic sequence
and novel sequences are often generated through mutagenesis
\citep{kinney-2010-using-deep-sequencing,currin-2015-synthetic-biology-directed}. Therefore,
applying an upper confidence bound procedure to these tasks
must consider the uncertainty inherent to the generation of new
sequences. We considered a mutagenesis library design problem
where five positions would be randomized with one of four
mutation rates and every possible length five DNA sequence is
used as a starting point for randomization (Fig.
\ref{fig:org82b2f40}A,B). In this case, \(\theta = \{s, \mu\}\) where \(s \in
       \{\text{A,C,T,G}\}^5\) is the starting DNA sequence and \(\mu\) is
the mutation rate.

We developed a simulation of bacterial promoter design using a
published model of the \emph{Escherichia coli} \emph{lac} promoter, which
models the expression levels as a function of promoter sequence
with linear and quadratic terms corresponding to individual
nucleotide and position interactions, respectively
\citep{otwinowski-2013-genotype-to-phenotype}. We use this
published model as an oracle for simulating iterative design of
the \emph{lac} promoter at the regulatory targets of two proteins,
CRP and RNAP.  Each regulator targets two regions, which
contained the largest linear and quadratic terms yielding a
diverse fitness landscape on which to optimize (Figs
\ref{fig:org1e1534a},
\ref{fig:orged6750c}). 

Probabilistic modeling of sequence to phenotype is an active
area of research \citep{riesselman-2018-deep-generative-models}.
We adapted a previous approach of modeling protein fitness with
GPs and a linear kernel to predict the expression levels of
the \emph{lac} promoter as a function of the promoter sequence
\citep{romero-2012-navigating-protein-fitness}. While relatively simple,
we found that this model was able to capture relevant global
trends in the data and would therefore provide a reasonable
test of our algorithm's performance (Fig.
\ref{fig:org3e2ce6c}), with the added advantage of
providing a straight-forward GP model to use for SSBO.

We applied batch SS-GPUCB to each of the four \emph{lac} promoter
regions for each acquisition function for a batch size of
five. We found that for each region, batch SS-GPUCB outperforms
all other methods (Fig. \ref{fig:org82b2f40}C). We further expect that the
difference in performance between batch SS-GPUCB and other
methods would grow considerably as the size of each batch and
sequence space are increased to reflect the sizes commonly seen
in iterative genetic sequence design
\citep{currin-2015-synthetic-biology-directed}.

\section{Conclusion}
\label{sec:orgeb0f6a3}
Stochastic sampling is common in scientific and
engineering domains. We have provided the theoretical groundwork
to enable broad application of BO techniques to optimization tasks
with stochastic sampling, with proof of sub-linear regret
bounds. Our empirical results suggest that this method will be
successful in a broad range of applications, and enhance the use of
BO in real-world optimization scenarios.

\subsubsection*{Acknowledgments}
\label{sec:orgb7efc86}
Funding for this project came from the NIST Innovation in
Measurement Science Grant: Genetic Sensor Foundry for Predictive
Engineering of Living Measurement Systems.  PDT is additionally
funded by the NRC Postdoctoral Fellowship. Analysis was
performed on the NIST Enki HPC cluster.  The identification of
any commercial product or trade name does not imply endorsement
or recommendation by the NIST, nor is it intended to imply that
the materials or equipment identified are necessarily the best
available for the purpose.

\bibliography{references}

\newpage

\appendix

\counterwithin{figure}{section}
\counterwithin{table}{section}

\section{Appendix}
\label{sec:org52daaf2}

\subsection{Precursors}
\label{sec:org747bef3}
This section establishes generally useful properties for the
following proofs. The proofs for discrete and continuous
distributions are in regards to the cumulative expected regret at
iteration \(T\) for each \(\theta_t\) \(1 \leq t \leq T\): R\(_{\text{T}}\) =
\(\sum_{\text{i=1}}^{\text{T}}\) r\(_{\pi}\)(\(\theta_{\text{t}}\)). For notational simplicity, we write
\(r_t\) in place of \(r_\pi(\theta_t)\).

\begin{lemma}
Let $x$ belong to a set $D$. 
Define $\sigma^2_{t-1}$ as the predictive variance of a GP with kernel $k(x, x) \leq 1$ and observation variance $\sigma^2$ trained on $t-1$ observations.
Then,
\begin{align*}
\sum_{i=1}^T \frac{1}{2}\log(1 + \sigma^{-2} \sigma^2_{t-1}(x)) = I(Y_T; f)
\end{align*}
Where $Y_T$ corresponds to $T$ observations selected by SS-GPUCB, and $I(Y_T; f)$ is the mutual information between observations $Y_T$ and $f$.
\end{lemma}
\begin{proof}
The proof is adapted from \citep{srinivas-2012-information-theoretic-regret}, Lemma 5.3 but assuming a single $x \in D$. As we show below, our regret bounds can be stated with $x$ held constant and the expectation taken over each $\theta_t$.
First, note that
\begin{align*}
I(Y_T; f) = H(Y_T) - H(Y_T \vert f).
\end{align*}
We have, for a single observation $y_T$:
\begin{align}
H(y_T \vert f) = \frac{1}{2} \log ( 2 \pi e \sigma^2 )
\end{align}
for all $t$ because $y_t$ given $f$ is a normal variate.
We also have
\begin{align*}
H(Y_T) = & H(Y_{T-1}) + H(y_t \vert Y_{T-1}) \\ 
       = & H(Y_{T-1}) + \log(2 \pi e (\sigma^2 + \sigma_{t-1}^2(x)))/2 \\
       = & \frac{1}{2} \sum_{i=1}^T \log\Big(2 \pi e (\sigma^2 + \sigma_{t-1}^2(x))\Big)
\end{align*}
The second equation follows from the fact that $\sigma_{t-1}^2$ does not depend on the values of $Y_T$.
Finally we have
\begin{align*}
I(Y_T; f) = & \sum_{i=1}^T \Big[\frac{1}{2} \log\Big(2 \pi e (\sigma^2 + \sigma_{t-1}^2(x))\Big) - \frac{1}{2} \log ( 2 \pi e \sigma^2 ) \Big]\\
	  = & \sum_{i=1}^T \frac{1}{2} \log(1 + \sigma^{-2} \sigma^2_{t-1}(x))       
  \end{align*}
\end{proof}

From here we establish a useful bound on the sum of variance
terms used in the following proofs.
\begin{lemma}
\label{lemma:sum-variance}
Take a series over the variable $t$, $1 \leq t \leq T$. 
Suppose that the $t$ dependent variable $\beta_t$ is non-decreasing.
Additionally, let there be an observation $x \in D$ for each iteration $t$.
Let $\sigma_{t-1}(x)$ be the predictive variance at iteration $t-1$ of a GP with kernel $k$ such that $k(x, x) \leq 1$ for all $x$ and noise variance $\sigma^2$.
Let $C_1 = \frac{8}{log(1+\sigma^{-2})}$.
Then,
\begin{align*}
 \sum_{t=1}^T 4 \beta_t \sigma^2_{t-1}(x) \leq \beta_T \gamma_T C_1
\end{align*}
\end{lemma}
\begin{proof}
First, we have
\begin{align*}
 4 \beta_t \sigma^2_{t-1}(x) \leq 4 \beta_T \sigma^2_{t-1}(x)
 \end{align*}
because $\beta_t$ is non-decreasing.
Next,
\begin{align*}
 4 \beta_T \sigma^2_{t-1}(x) \leq & 4 \beta_T \sigma^2 \sigma^{-2} \sigma^2_{t-1}(x) \\
			     \leq & 4 \beta_T \sigma^2 C_2 \log(1 + \sigma^{-2} \sigma_{t-1}(x))
 \end{align*}
 where $C_2 = \frac{\sigma^{-2}}{\log(1+\sigma^{-2})}$.
 This is due to the fact that $s^2 \leq C_2 \log(1 + s^2)$ for $s^2 \in [0, \sigma^2]$ and $\sigma^{-2} \sigma_{t-1}^2(x) \leq \sigma^{-2} k(x, x) \leq \sigma^{-2}$.
 Then, due to the fact that $C_2 = \frac{C_1}{8 \sigma^2}$, we have
 \begin{align*}
 4 \beta_T \sigma^2 C_2 \log(1 + \sigma^{-2} \sigma_{t-1}(x)) \leq \beta_T C_1 \Big[\frac{1}{2} \log(1 + \sigma^{-2}\sigma_{t-1}(x))\Big]
 \end{align*}
 Therefore we have,
 \begin{align}
 \sum_{t=1}^T 4 \beta_t \sigma^2_{t-1}(x) \leq \beta_T \gamma_T C_1
 \label{eq:info-gain-bound}
\end{align}

\end{proof}

As described in the main text, for sampling distributions \(\pi(x
     \vert \theta)\), we define a constant
\begin{align*}
\pi^* := \max_{x \in D, \theta \in \Theta} \pi(x \vert \theta).
\end{align*}
\(\pi^*\) corresponds to the maximal pdf value of \(\pi\) for all
possible \(x\) and \(\theta\) in the optimization problem. This
constant is useful for bounding the expected pdf value of \(\pi\)
for a given \(\theta\), which arises in our proofs. Specifically,
\begin{align}
\label{eq:pistar-discrete}
E_{\pi(\theta)}[\pi(\theta)] = \sum_{x \in D} \pi^2(x \vert \theta) \leq \pi^* \sum_{x \in D} \pi(x \vert \theta) \leq \pi^*
\end{align}
for the discrete case, and
\begin{align}
\label{eq:pistar-continuous}
E_{\pi(\theta)}[\pi(\theta)] = \int \pi^2(x \vert \theta) dx \leq \pi^* \int \pi(x \vert \theta) dx \leq \pi^*
\end{align}
in the continuous case.
Using these terms, we now state the main theoretical
results of this work.

\subsection{Discrete distribution}
\label{sec:org2deb103}

Here, we consider a sample space, and sampling distribution, with
finite dimensionality. Specifically, \(x \in D\) and \(\vert D \vert
     < \infty\). Each sampling distribution is then well defined on
this space, \(\pi(x \vert \theta) > 0\) \(\forall x \in D\) and
\(\sum_{x \in D} \pi(x \vert \theta) = 1\) \(\forall \theta \in
     \Theta\). The proof for this case follows similarly to that of the
finite dimensional case of the original GP-UCB paper
\citep{srinivas-2012-information-theoretic-regret}.

\begin{theorem}
Let $\delta \in \{0, 1\}$ and $\beta_t = 2 \log(\vert D \vert \pi^2 / 6 \delta)$.
Then the regret associated with performing SS-GPUCB has the following probabilistic bound:
\begin{align}
  Pr \Big( R_T \leq \sqrt{T C_1 \beta_T \gamma_T \vert D \vert \pi^*} \Big) \geq 1 - \delta
\end{align}
\end{theorem}
Proof to follow.

\begin{lemma}[\citep{srinivas-2012-information-theoretic-regret} Lemma 5.1]
\label{lemma:dist-bound}
  Pick $\delta \in (0, 1)$ and set $\beta_t = 2 \log(\vert D \vert \pi_t / \delta)$, 	
  such that $\sum_{t \geq 1} 1/\pi_t = 1$ and $\pi_t > 0$. Then,
  \begin{align}
   \vert f(x) - \mu_{t-1}(x) \vert  \leq \beta_t^{1/2} \sigma_{t-1}(x) \quad \forall x \in D, \forall t \geq 1 
  \end{align}
  holds with probability $\geq 1 - \delta$.
\end{lemma}
\begin{proof}
  See \citep{srinivas-2012-information-theoretic-regret} Lemma 5.1.
\end{proof}

\begin{lemma}
  \label{lemma:discrete_regret_iter}
  Fix $t \geq 1$. If $\vert f(x) - \mu_{t-1}(x) \vert \leq \beta_t^{1/2} \sigma_{t-1}(x)$ for all $x \in D$,
  then the expected regret $r_\pi(\theta_t) = \sum_{x \in D} f(x) \pi(x | \theta^*) - \sum_{x \in D} f(x) \pi(x | \theta_t)$ (Eq. \ref{eq:rpi_theta}) is bounded by $2\beta_t^{1/2} E_{\pi(x\vert \theta)}[\sigma_{t-1}(x)]$.
\end{lemma}
     \begin{proof}
       The proof is similar to that of \citep{srinivas-2012-information-theoretic-regret} Lemma 5.2, adapted to the expectation over $\pi(x \vert \theta)$.
       First, from the assumed bounds, we have for all $\theta \in \Theta$:
     \begin{align*}
       \sum_{x \in D} [f(x) - \mu_{t-1}(x)] \pi(x|\theta)| \leq \sum_{x \in D} \beta^{1/2} \sigma_{t-1}(x) \pi(x | \theta) .
     \end{align*}
     and therefore
     \begin{align}
       \sum_{x \in D} f(x) \pi(x|\theta) \leq & \sum_{x \in D} [\mu_{t-1}(x) + \beta^{1/2} \sigma_{t-1}(x)] \pi(x | \theta) \\
				    \leq & \sum_{x \in D} \alpha_t(x) \pi(x|\theta). 
       \label{eq:marginal_ucb}
     \end{align}

       Then, by definition of $\theta_t$ (Eq. \ref{eq:thetat}) and the above bounds, we have
       \begin{align*}
	\sum_{x \in D} f(x) \pi(x \vert \theta^*) \leq & \sum_{x \in D} \alpha_t(x) \pi(x \vert \theta^*) \\
						  \leq & \sum_{x \in D} \alpha_t(x) \pi(x \vert \theta_t) = E_{\pi(x \vert \theta_t)} [\alpha_t(x)]
       \end{align*}

       Therefore, we have
\begin{align*}
	 r_\pi(\theta_t) = & \sum_{x \in D} f(x) \pi(x | \theta^*) - \sum_{x \in D} f(x) \pi(x | \theta_t) \\
		      \leq & \sum_{x \in D} \alpha(x) \pi(x | \theta^*) - \sum_{x \in D} f(x) \pi(x | \theta_t) \\
		      \leq & \sum_{x \in D} \alpha(x) \pi(x | \theta_t) - \sum_{x \in D} f(x) \pi(x | \theta_t) \\
		      \leq &\sum_{x \in D} [\beta^{1/2} \sigma_{t-1}(x) + \mu_{t-1}(x) - f(x)] \pi(x | \theta_t) \\
		      \leq & \sum_{x \in D} 2 \beta^{1/2} \sigma_{t-1}(x) \pi(x | \theta_t) \\
		      \leq & 2 \beta^{1/2} E_{\pi(x \vert \theta_t)}[\sigma_{t-1}(x)].
\end{align*}
     \end{proof}

\begin{lemma}
  Set $\delta \in (0, 1)$ and $\beta_t$ as above. Then the following holds with probability $1 - \delta$:
  \begin{align*}
	  \sum_{t=1}^T (r_\pi(\theta_t))^2 \leq C_1 \beta_T \gamma_T |D| \pi^*. 
  \end{align*}
\end{lemma}
     \begin{proof}
       From Lemma \ref{lemma:discrete_regret_iter}, we have:
\begin{align*}
\sum_{t=1}^T (r_\pi(\theta_t))^2 \leq & 4 \sum_{t=1}^T \beta_t E^2_{\pi(x \vert \theta_t)}[\sigma_{t-1}(x)] = 4 \sum_{t=1}^T \beta_t \Big[\sum_{x \in D} \sigma_{t-1}(x) \pi(x \vert \theta_t)\Big]^2 \\
				 \leq & 4 \sum_{t=1}^T \Big(\beta_t \sum_{x \in D} \sigma_{t-1}^2(x) \sum_{x \in D}\pi^2(x \vert \theta_t)\Big) \\
\end{align*}
       where the second step comes from the Cauchy-Schwartz inequality. Using Eq. \ref{eq:pistar-discrete}, we have 
       \begin{align*}
       4 \sum_{t=1}^T \Big(\beta_t \sum_{x \in D} \sigma_{t-1}^2(x) \sum_{x \in D}\pi^2(x \vert \theta_t)\Big) \leq \pi^*\Big(4 \sum_{t=1}^T \beta_t \sum_{x \in D} \sigma_{t-1}^2(x) \Big).
       \end{align*}
       Next we adapt the maximal information gain bound developed in \citep{srinivas-2012-information-theoretic-regret}.
       Specifically, we have
\begin{align*}
 \sum_{t=1}^T 4 \beta_t \sum_{x \in D} \sigma_{t-1}(x)^2 = & \sum_{x \in D} \sum_{t=1}^T 4 \beta_t \sigma_{t-1}^2(x)  \\
								\leq & \sum_{x \in D} \beta_T C_1 I(y_T; f_T) \\
								\leq & \sum_{x \in D} \beta_T C_1 \gamma_T \\
								\leq & \beta_T C_1 \gamma_T \vert D\vert.
\end{align*}
       The first two steps follow those of Lemma \ref{lemma:sum-variance}.
       Combining terms we get the bounds as described.
     \end{proof}

From here, we use the fact that \(R_T^2 \leq T \sum_{t=1}^T
     r_t^2\) from the Cauchy-Schwartz inequality to establish that
\(R_T \leq \sqrt{T C_1 \beta_T \gamma_T \vert D \vert \pi^*}\).

\subsection{Continuous distribution}
\label{sec:org783e731}
We now consider the case of closed, bounded \(D \subset
     R^d\). Specifically, we consider \(D = [0, r]^d\), \(d \in
     \mathbb{N}\) and \(r > 0\). The volume of this set is then \(V = r^d\).

\begin{theorem}
Let $D \subset [0, r]^d$ with $V = r^d$.
Suppose that $f$ is drawn from 
a GP with kernel $k$ that satisfies the probability bound, for some constants $a, b > 0$:
\begin{align*}
  Pr \Big( \text{sup}_{x \in D} \vert \partial f / \partial x_j \vert > L \Big) \leq a e^{-L^2/b^2} \quad j = 1, \dots, d.
\end{align*}
Pick $\delta \in (0, 1)$ and set
\begin{align*}
  \beta_t = 2 \log (t^2 2 \pi^2 / (3 \delta)) + 2 d \log\Big( t^2 dbr \sqrt{\log(4da/\delta)} \Big).
\end{align*}
Then
\begin{align*}
Pr\Big( R_T \leq \sqrt{T \beta_T \gamma_T C_1 V \pi^*} + \pi^2/3 \Big) \geq 1 - \delta.
\end{align*}
\end{theorem}
Proof to follow.

\begin{lemma}[Adapted from \citep{srinivas-2012-information-theoretic-regret} Lemma 5.5]
\end{lemma}
Let \(x_t\) be the point sampled at step \(t\), given by \(x_t \sim
     \pi(x \vert \theta_t)\). Choose \(\delta \in (0, 1)\) and \(\beta_t
     = 2 \log(\pi_t / \delta)\), with \(\sum_{t \geq 1} \pi^{-1}_t = 1\) and
\(\pi_t > 0\). Then,
\begin{align*}
\vert f(x_t) - \mu_{t-1}(x_t) \vert \leq \beta_{t-1}^{1/2} \sigma_t(x_t) \quad \forall t \geq 1
\end{align*}
holds with probability \(1 - \delta\).
\begin{proof}
see \citep{srinivas-2012-information-theoretic-regret} lemma 5.5. 
\end{proof}

For sake of analysis, define a discretization of \(D\): \(D_t
     \subset D\), \(\vert D_t \vert < \infty\).
\begin{lemma}[Adapted from \citep{srinivas-2012-information-theoretic-regret} Lemma 5.6]
\label{lemma:dt-max-dist}
Pick $\delta \in (0, 1)$ and set $\beta_t = 2 log (\vert D_t \vert \pi_t / \delta)$, with $\sum 1 / \pi_t = 1$ and $\pi_t > 0$.
Then,
\begin{align*}
Pr \Big(\vert f(x) - \mu_{t-1}(x) \vert \leq \beta_t \sigma_{t-1}(x) \quad \forall x \in D_t, \forall t \geq 1 \Big) \geq 1 - \delta.
\end{align*}
\end{lemma}
\begin{proof}
Same as Lemma \ref{lemma:dist-bound}, replacing finite $D$ with $D_t$.
\end{proof}

Now assume that
\begin{align*}
Pr\big(\forall j, \forall x \in D, \vert \partial f /  \partial x_j \vert < L \big) \geq 1 - da \text{exp}(-L^2/b^2).
\end{align*}
From which follows
\begin{align}
\label{eq:fx-diff-bound}
Pr\Big(\forall x, x' \in D, \vert f(x) - f(x') \vert < L \vert\vert x - x'\vert\vert_1 \Big) \geq 1 - da \text{exp}(-L^2/b^2).
\end{align}

Next, set the size of \(D_t\) to \(\tau_t^d\) such that for all \(x
     \in D\),
\begin{align*}
\vert \vert x - [x]_t\vert \vert_1 \leq r d / \tau_t,
\end{align*}
where \([x]_t\) corresponds to the closest point in \(D_t\) to \(x\).
This can be accomplished by placing \(\tau_t\) equally spaced coordinates
in each dimension.

\begin{lemma}[Adapted from \citep{srinivas-2012-information-theoretic-regret} Lemma 5.7]
Pick $\delta \in (0, 1)$ and set $\beta_t = 2 \log(2 \pi_t / \delta) + 4d\log(dtbr \sqrt{\log(2da/\delta)})$
where $\sum \pi_t^{-1} = 1$ and $\pi_t > 0$. Let $\tau_t = d t^2 b r \sqrt{\log(2 d a / \delta)}$. Then
\begin{align}
Pr\big( \vert f(x) - \mu_t([x]_t) \vert \leq \beta_t^{1/2} \sigma_t([x]_t) + \frac{1}{t^2}; \forall t, \forall x \in D \big) > 1 - \delta.
\end{align}
\end{lemma}
\begin{proof}
From Eq. \ref{eq:fx-diff-bound} we have
\begin{align*}
Pr\big( \forall x, x' \in D \vert f(x) - f(x') \vert \leq b \sqrt{\log(2da/\delta)} \vert \vert x - x' \vert \vert_1\big) \geq 1 - \delta/2.
\end{align*}
Then,
\begin{align*}
Pr\big( \vert f(x) - f([x]_t) \vert \leq \frac{r d}{\tau_t} b \sqrt{\log(2da/\delta)} \big) \geq 1 - \delta/2.
\end{align*}
Therefore, with the selected value of $\tau_t$, we have
\begin{align*}
Pr\big( \vert f(x) - f([x]_t) \vert \leq \frac{1}{t^2}; \forall x \in D \big) \geq 1 - \delta/2.
\end{align*}

With the specified value of $\tau_t$, we have $\vert D_t \vert = \Big(d t^2 b r \sqrt{\log(2 d a / \delta)} \Big)^d$.
Then, with probability $\delta/2$ substituted into Lemma \ref{lemma:dt-max-dist}, the result follows as described.
\end{proof}

In order to minimize the expected regret, we now adapt the UCB
calculation to marginalize within the set \(D_t\).
Specifically, find \(\theta_t\) as
\begin{align*}
\theta_t := \argmax_{\theta \in \Theta} \int \alpha([x]_t) \pi(x \vert \theta) dx.
\end{align*}
First, from this definition of \(\theta_t\) and the previous lemma
we have with probability \(\geq 1-\delta\),
\begin{align*}
\int f(x) \pi(x \vert \theta^*) dx \leq & \int \Big(\mu_t([x]_t) + \beta_t^{1/2} \sigma_t([x]_t) + \frac{1}{t^2}\Big) \pi(x \vert \theta^*) dx \\
				   \leq & \int \Big(\alpha([x]_t) + \frac{1}{t^2}\Big) \pi(x \vert \theta^*) dx \\
				   \leq & \int \Big(\alpha([x]_t) + \frac{1}{t^2}\Big) \pi(x \vert \theta_t) dx \\
				   \leq & \int \Big(\mu_t([x]_t) + \beta_t^{1/2} \sigma_t([x]_t) + \frac{1}{t^2}\Big) \pi(x \vert \theta_t) dx.
\end{align*}
Then, we have
\begin{align*}
r_\pi(\theta_t) = & \int f(x) \pi(x \vert \theta^*) dx - \int f(x) \pi(x \vert \theta_t) dx \\
		\leq & \int \Big( \beta_t^{1/2} \sigma_t([x]_t) + \mu_t([x]_t) - f(x) + \frac{1}{t^2} \Big) \pi(x \vert \theta_t) dx \\
		\leq & \int \Big( 2\beta_t^{1/2} \sigma_t([x]_t) + 2\frac{1}{t^2} \Big) \pi(x \vert \theta_t) dx \\
		\leq & \frac{2}{t^2} + \int 2\beta_t^{1/2} \sigma_t([x]_t) \pi(x \vert \theta_t) dx
\end{align*}

Then, bound the second term similar to lemma \ref{lemma:discrete_regret_iter} for the
discrete case using the Cauchy-Schwartz inequality.
First, we have
\begin{align}
\sum_{i=1}^T \Big(\int (2 \beta^{1/2} \sigma_t([x]_t)) \pi(x \vert \theta_t) dx \Big)^2 
     	   \leq & \sum_{i=1}^T \int (4 \beta \sigma^2_t([x]_t)) dx \int \pi^2(x \vert \theta_t) dx  & \quad \text{via Cauchy-Schwartz} \\
	   \leq & \pi^* \sum_{i=1}^T \int (4 \beta \sigma^2_t([x]_t)) dx & \quad \text{via Eq. \ref{eq:pistar-continuous}} \\
	   \leq & \pi^* \int \sum_{i=1}^T (4 \beta \sigma^2_t([x]_t)) dx & \\
	   \leq & \pi^* V \beta_T \gamma_T C_1& 
	   \label{eq:continuous-sum-int-bound}
\end{align}
Returning to the expected regret, we have
\begin{align*}
   R_T = \sum_{t=1}^T r_\pi(\theta_t) = & \sum_{i=1}^T \Big(\int (2 \beta^{1/2} \sigma_t([x]_t)) \pi(x \vert \theta_t) dx \Big) + \pi^2/3 & \text{because } \sum 1 / t^2 = \pi^2 / 6 \\
	      		  \leq & \sqrt{T \beta_T \gamma_T C_1 V \pi^*} + \pi^2/3 & \text{via Eq. \ref{eq:continuous-sum-int-bound} and Cauchy-Schwartz}
\end{align*}

\begin{figure}[btp]
\centering
\includegraphics[width=\textwidth]{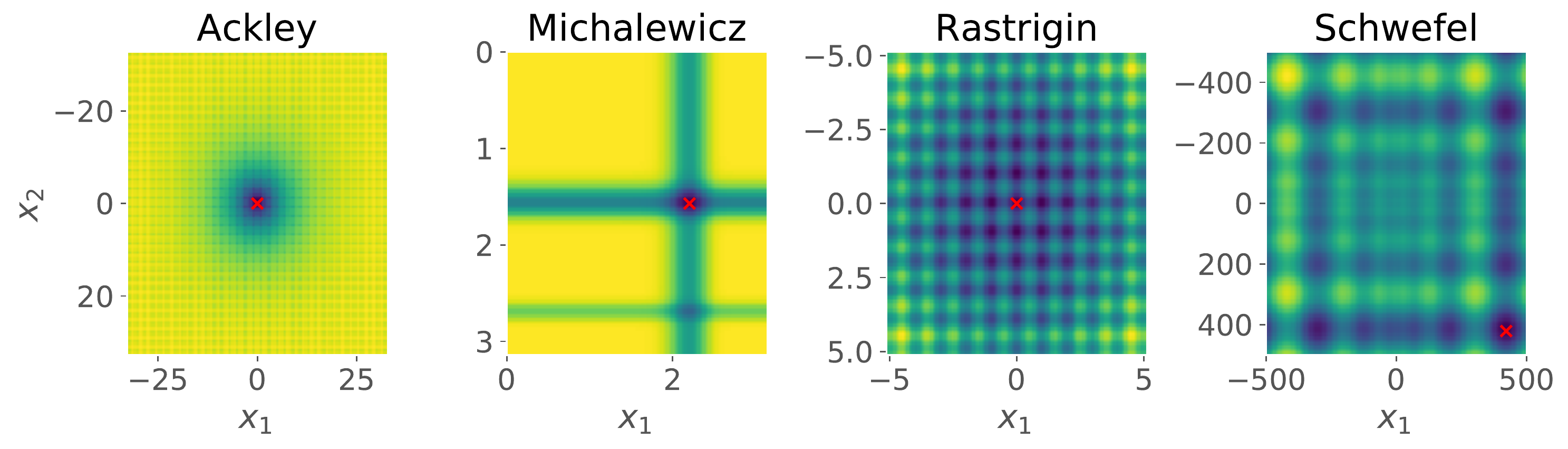}
\caption{\label{fig:orgc22416b}
\textbf{Objective functions used for simulations.} 2-dimensional views of the objective functions used in this study. These functions were taken from standard bench-marking literature on optimization processes. These objectives have various useful properties for assessing optimization techniques, including local optima, multiple periodicities and magnitudes, and sharp boundaries.}
\end{figure} 

\begin{table}[tb]
\caption{\label{tab:org07a86ca}
Objective functions. \(d=2\) in all cases.}
\centering
\begin{tabular}{l|l|c}
Function & Equation & range\\
\hline
Ackley & \(-20 \text{exp} \Big(-0.2 \sqrt{\frac{1}{d} \sum_{i=1}^d x_i^2 }\Big) - \text{exp}\big(\frac{1}{d} \sum_{i=1}^d \text{cos}(2\pi x_i)\big) + 20 + \text{exp}(1)\) & \(-32.768 \leq x \leq 32.768\)\\
Michaelwicz & \(-\sum_{i=1}^d \text{sin}(x_i) \text{sin}^{2m} \big( \frac{i x_i^2}{\pi} \big)\) & \(0 \leq x \leq \pi\)\\
Rastrigin & \(10d + \sum_{i=1}^d \big( x_i^2 - 10 \text{cos}(2\pi x_i) \big)\) & \(-5.12 \leq x \leq 5.12\)\\
Schwefel & \(418.9829 d - \sum_{i=1}^d x_i \text{sin} (\sqrt{\vert x_i \vert})\) & \(-500 \leq x \leq 500\)\\
\end{tabular}
\end{table}

\begin{figure}[btp]
\centering
\includegraphics[width=\textwidth]{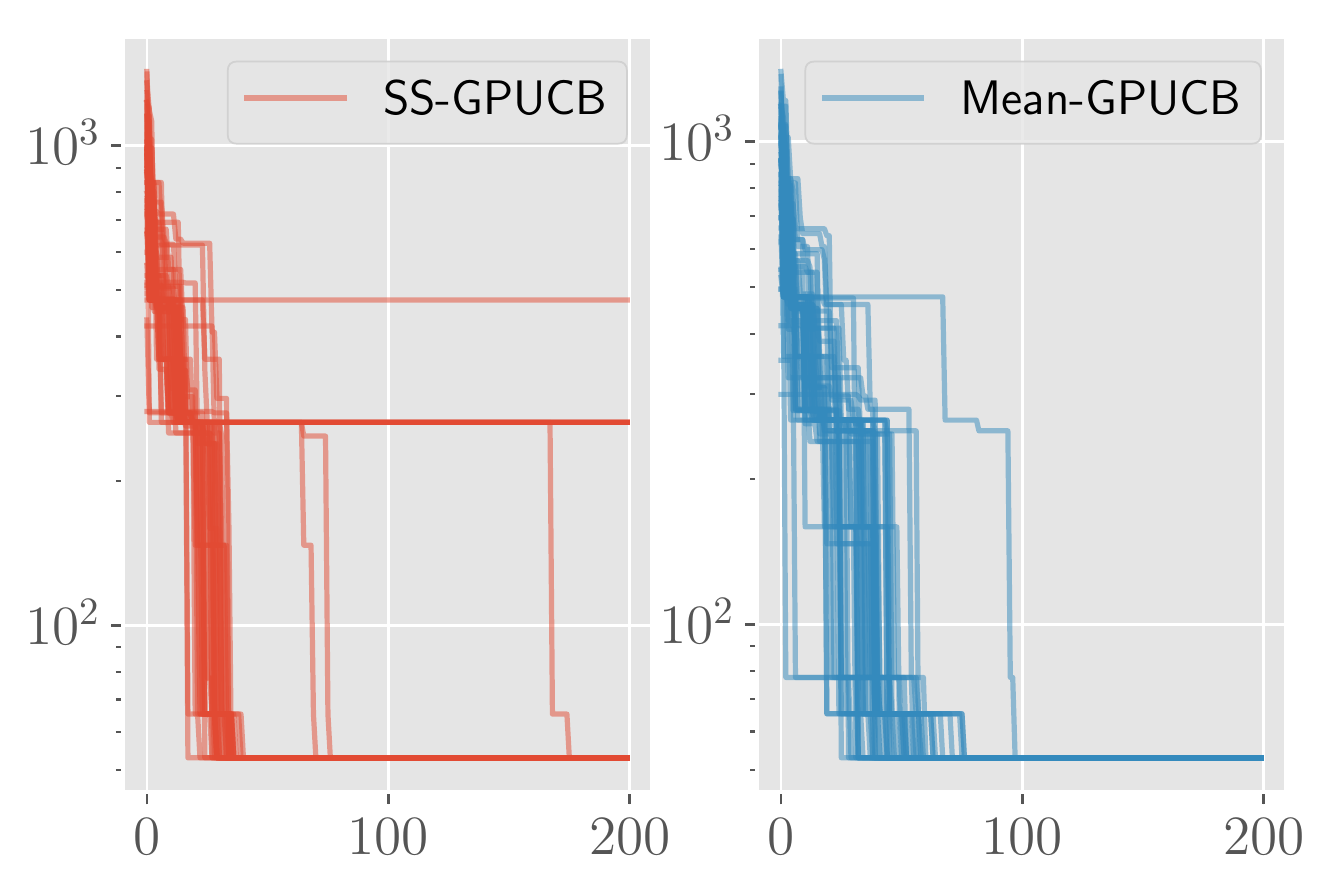}
\caption{\label{fig:org181cb08}
\textbf{Simple regret trace plots for Schwefel function.} Simple regret trace plots for sequential stochastic optimization for SS-GPUCB and Mean-GPUCB}
\end{figure} 

\subsection{Exact marginalization of acquisition function}
\label{sec:org8bc60d9}
As in \citet{desautels-2014-parallelizing-exploration-exploitation},
define a feedback function \(\text{fb}[t]\) specifying the
number of observations available at iteration \(t\). For
example, if each batch is of fixed size \(B\), then \(\text{fb}[t] =
     \lfloor (t-1)/B \rfloor B\). Normal sequential feedback
corresponds to \(\text{fb}[t] = t - 1\).
Define
\begin{align}
   \hat{\alpha}_t(x) = \mu_{\text{fb}[t]}(x) + \beta^{1/2} \sigma_{t-1}(x),
\end{align}
be the approximation of \(\alpha_t(x)\) as suggested in
\citep{desautels-2014-parallelizing-exploration-exploitation}. The
predictive mean is computed from the available feedback at a
given iteration \(t\) (e.g. \(\text{fb}[t]\) data points) while
\(\sigma_t(x)\) can be calculated exactly because it does not
depend on the observations. 

At time \(t\), define the expected UCB for a batch of size \(B\) as
\begin{align}
\label{eq:batch-expected-regret}
E_{\pi(\theta)} \Bigg[\sum_{i=1}^B \alpha_{t+i}(x)\Bigg] = & \sum_{i=1}^B E_{\pi(\theta)}[\alpha_{t+i}(x) \vert x_{1:t+i-1}] \\
     		     				     \approx & \sum_{i=1}^B E_{\pi(\theta)}[\hat{\alpha}_{t+i}(x)\vert x_{1:t+i-1}],
\end{align}
where \(x_{1:t+i-1} = \{x_1, \dots x_{t+i-1}\}\) are the location
of observations \(1\) to \(t+i-1\). The issue with this calculation
arises from the term \(\sigma_{t-1}(x)\) in \(\hat
     \alpha_t(x)\). This term depends on previous observations in the
batch \(x_{t:t+i-1}\) and therefore requires an expectation
calculation over the inverse of the kernel matrix including these
terms. For exact calculation, this requires a combinatorial
calculation over each possible \(x_i\), \(t \leq i \leq t + i -
     1\). Alternatively, Monte-carlo methods can be employed
\citep{snoek-2012-practical-bayesian-optimization}.

\subsection{Batch penalty calculation}
\label{sec:orgd4314a0}
We use the local penalty calculation of
\citep{gonzalez-2015-batch-bayesian-optimization}, which is
constructed with the assumption of Lipschitz continuity on the
optimization target \(f\). Specifically there is some \(L\) such that
\begin{align}
  \vert f(x_1) - f(x_2) \vert \leq L \vert \vert x_1 - x_2 \vert \vert, \quad \forall x_1, x_2 \in D.
\end{align}
This constant is estimated from the current model as
\begin{align}
  \hat{L} = \max_{x \in D, i \in k} \vert \vert \frac{d}{d x^{(i)}} \mu(x)\vert \vert,
\end{align}
where \(\frac{d}{d x^{(i)}}\) is the derivative with respect to
input dimension \(i\) as is directly calculated from the GP model
\citep{solak-2003-derivative-gaussian}. This term is then combined
with \(\hat{M}\), the approximate maximum of the function
calculated from all current observations, the local penalty
function \(\varphi(x_j; x_k)\) for previously selected observation position \(x_k\) is
\begin{align}
  \varphi(x_j; x_k) = \frac{1}{2} erfc\Bigg(- \frac{\hat{L} \vert \vert x_j - x_k \vert \vert -\hat{M} + \mu_n(x_j)}{\sqrt{2 \sigma^2_n(x_j)}} \Bigg),
\end{align}
where \(erfc\) is the complementary error function, and
\(\mu_n(\cdot)\) and \(\sigma_n^2(\cdot)\) are the predictive mean and
variance respectively of the GP model at iteration \(n\).

\begin{figure}[btp]
\centering
\includegraphics[width=\textwidth]{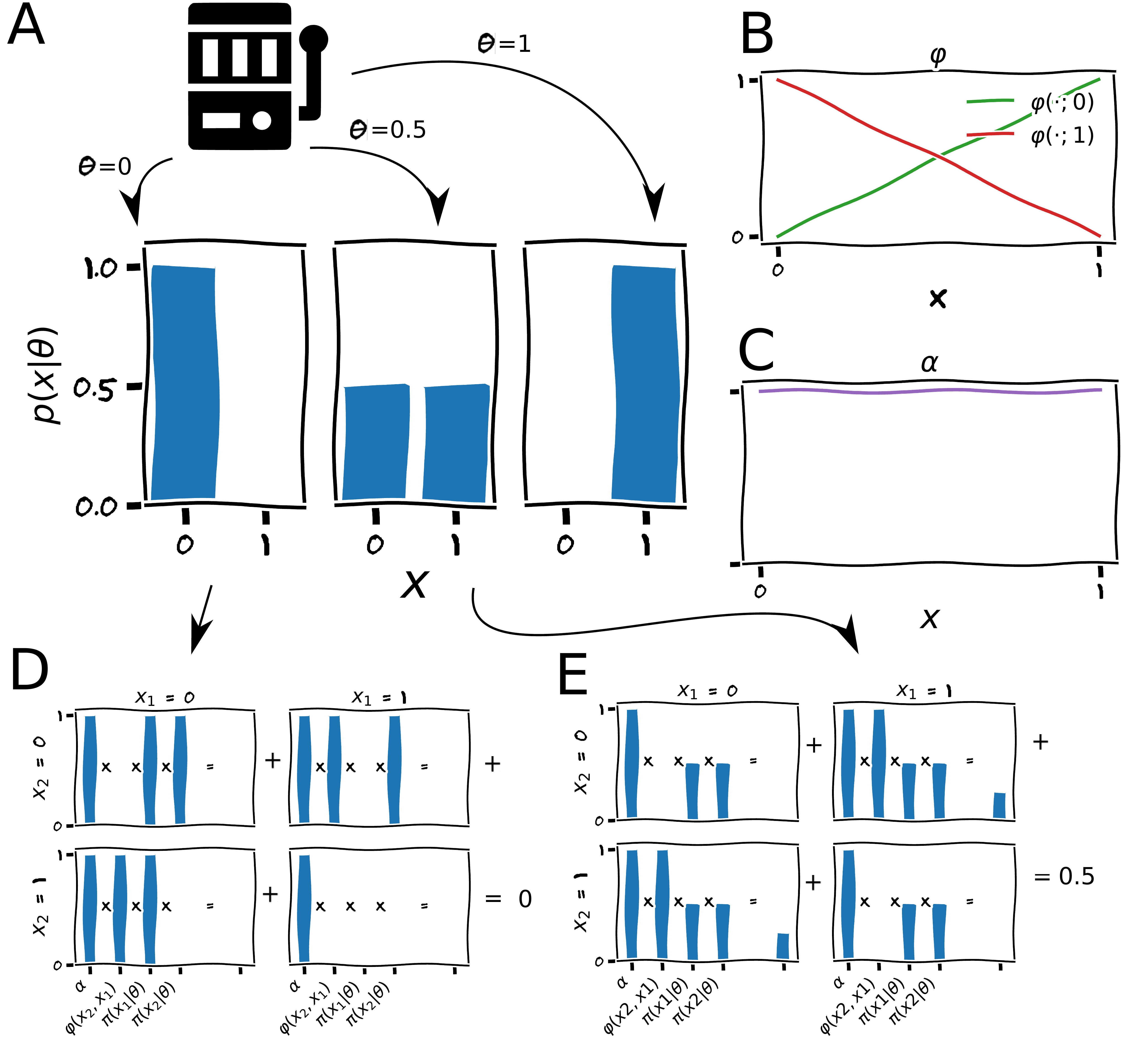}
\caption{\label{fig:org70f40a0}
\textbf{Simplified example of batch penalty usage.}. A simple three armed bandit (A) with parameter space \(\Theta = \{0, .5, 1\}\). We consider the acquisition value for each possible \(\theta\) when sampling two values \(x_1\) and \(x_2\). We simplify the acquisition penalty function \(\varphi\) to be \(1\) when the two \(x\) values differ and \(0\), otherwise (B). We also assume no previous data exists for this model, so the acquisition for both outcomes of \(x\) are equal (C). The total acquisition value for a specific \(\theta\) is the expected acquisition of both \(x_1\) and \(x_2\): \(E_{\pi(\theta)}\Big[ \alpha(x_1) + \alpha(x_2) \varphi(x_2; x_1) \Big]\). The first term simplifies to \(1\) for each value of \(\theta\) but the second term must be evaluated for each pair of \(x_1\) and \(x_2\) (D,E). For \(\theta=0\), there is no non-zero term in this portion of the expectation (D). For \(\theta=0.5\) however, the case where \(x_1\) and \(x_2\) differ occurs 50 \% of the time, leading to a total acquisition value of 0.5 from the second sample in the batch (E). In total, the acquisition for \(\theta=0\) (or \(=1\)) is \(1\) while for \(\theta=0.5\) it is \(1.5\). Therefore, applying the penalty \(\varphi\) allows the algorithm to properly weight the overlap in information between repeated sampling from the same distribution \(\pi(\theta)\).}
\end{figure}

\begin{figure}[tbp]
\centering
\includegraphics[width=\textwidth]{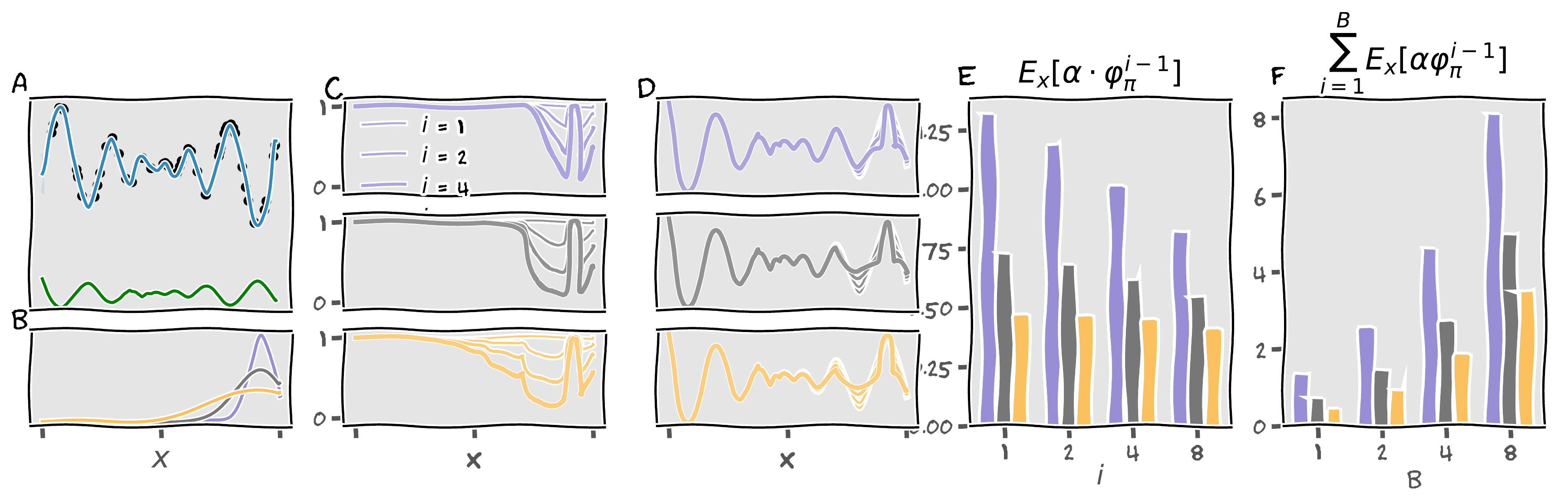}
\caption{\label{fig:org151d5bf}
\textbf{Batch penalty calculations during late iterations.} (A) GP model (blue) during late-stage optimization (e.g. many observations) and associate acquisition function (green). (B) Sampling distributions.(A) GP model (blue) during late-stage optimization (e.g. many observations) and associate acquisition function (green). (C) Expected penalty \(\varphi\) for each sampling distribution and varying iteration \(i\). (D) Expected acquisition penalized by \(\varphi\) for each distribution and varying iteration \(i\). (E) Expected acquisition for iteration \(i\) and (F) batch size \(B\).}
\end{figure}

\begin{figure}[htbp]
\centering
\includegraphics[width=\linewidth]{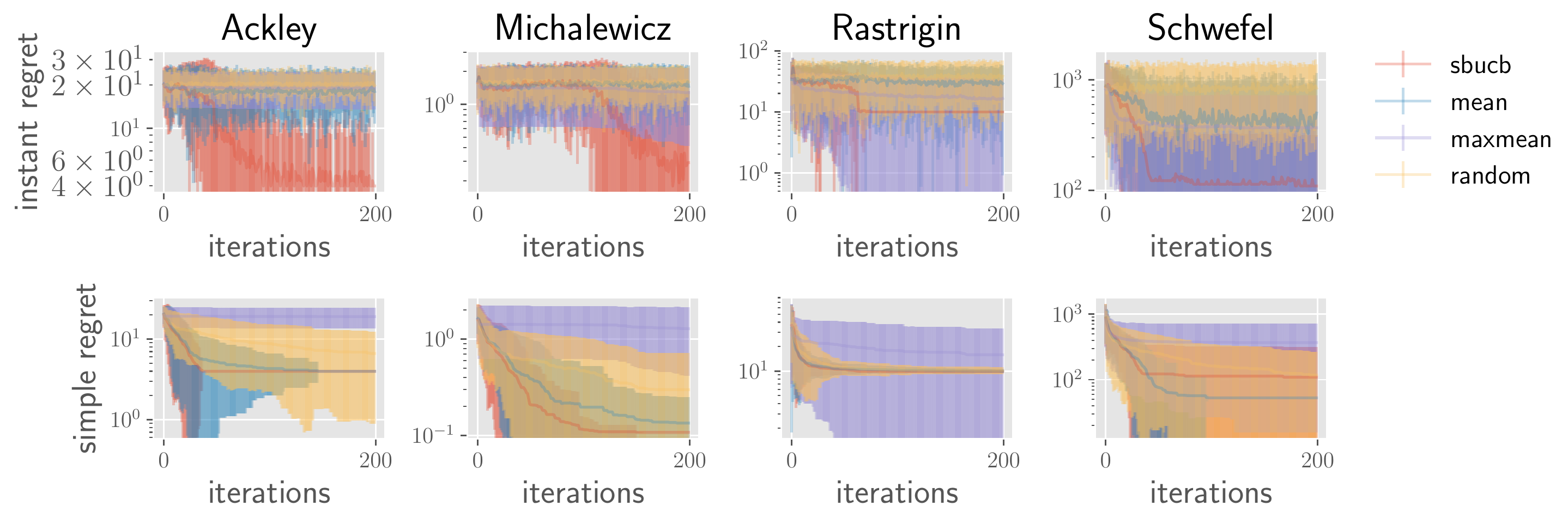}
\caption{\label{fig:org9bf9efb}
\textbf{Error on sequential regret.} 95 \% interval of regret for sequential optimization.}
\end{figure}

\begin{figure}[h]
\centering
\includegraphics[width=\linewidth]{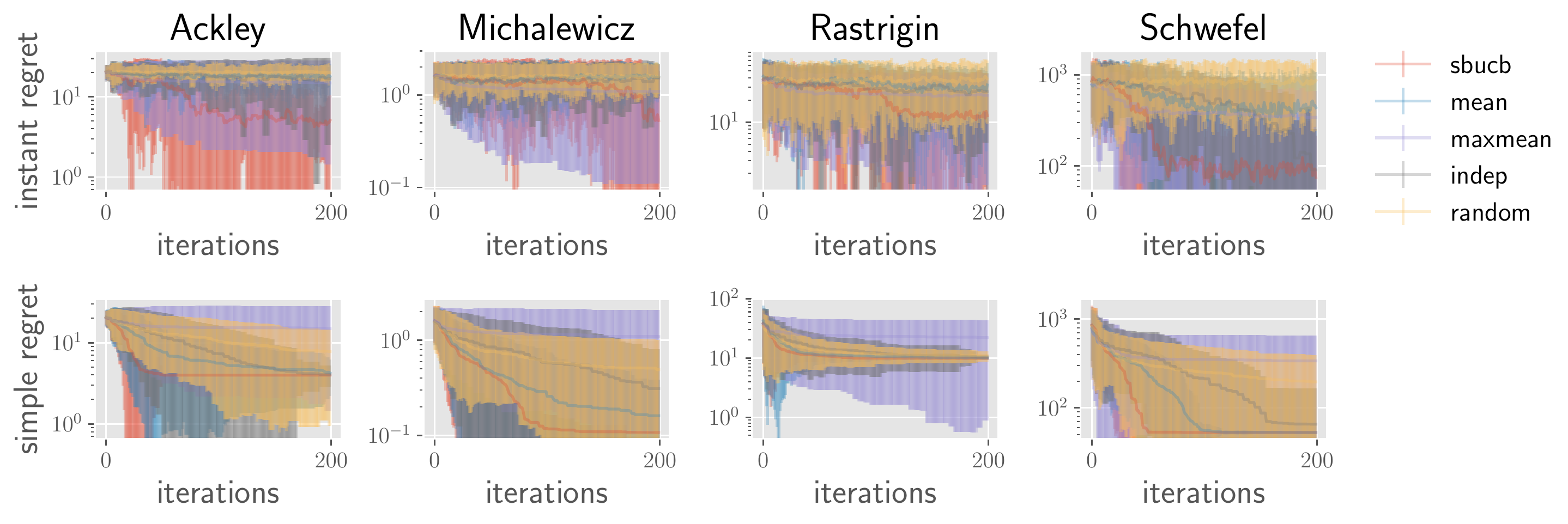}
\caption{\label{fig:orga073835}
\textbf{Error on batch regret} 95 \% interval of regret for batch optimization.}
\end{figure}

\begin{figure}[h]
\centering
\includegraphics[width=\linewidth]{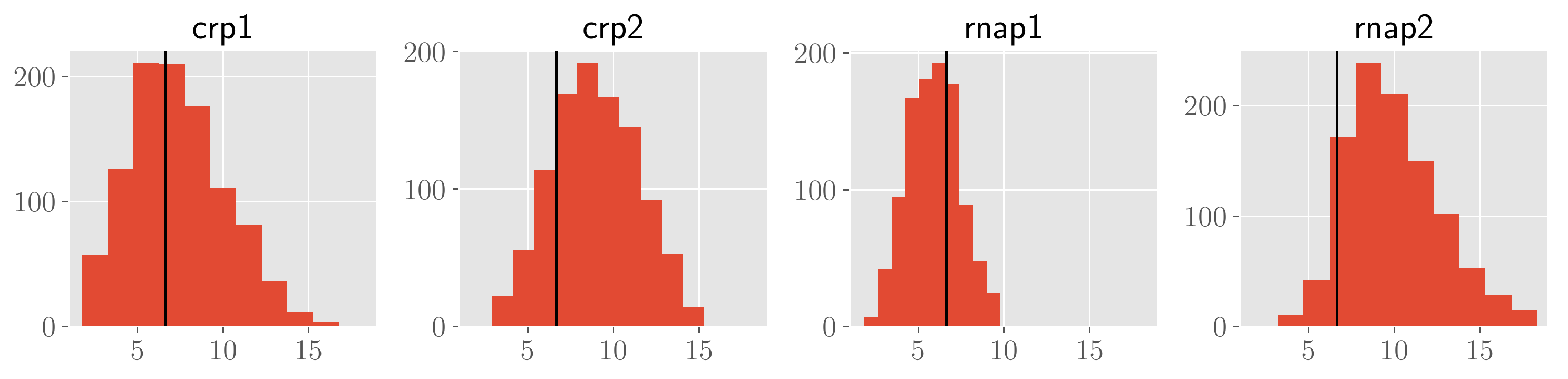}
\caption{\label{fig:org1e1534a}
\textbf{Distribution of LacI promoter fitness for various randomized regions.} Distribution of fitness values for each promoter region. Black line is wild-type fitness.}
\end{figure}

\begin{figure}[h]
\centering
\includegraphics[width=\linewidth]{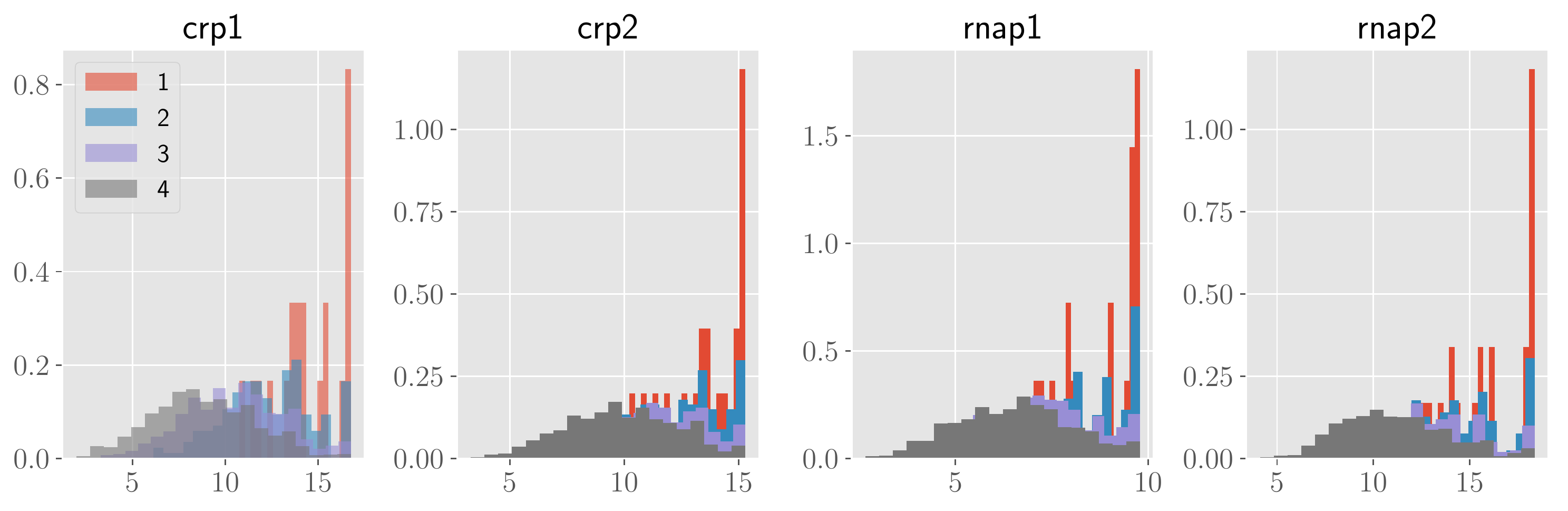}
\caption{\label{fig:orged6750c}
\textbf{Local optima landscape} Local fitness landscape for each promoter region. The distribution of fitness values with hamming distance (1---4) from the optima.}
\end{figure}

\begin{figure}[h]
\centering
\includegraphics[width=\linewidth]{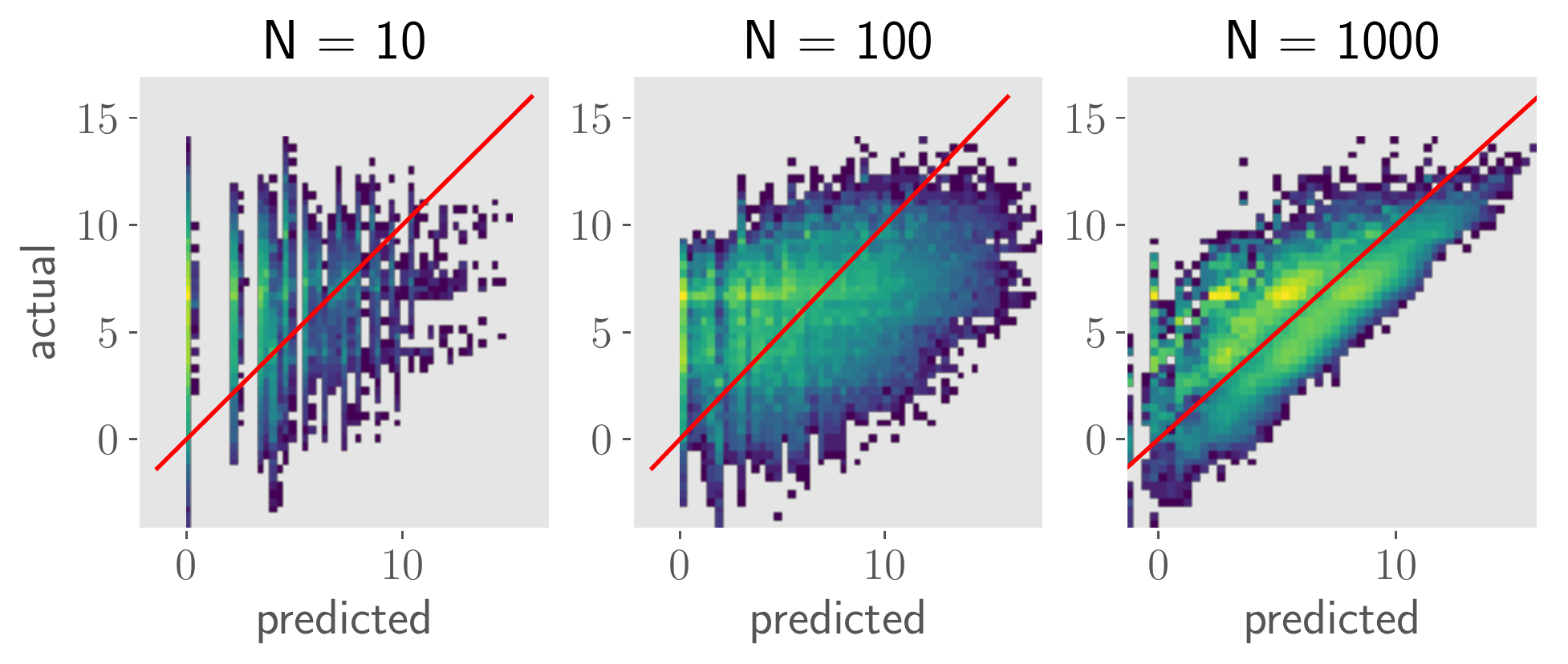}
\caption{\label{fig:org3e2ce6c}
\textbf{GP model prediction of LacI binding for increasing sample size.} GP model predictive mean vs actual fitness for models trained with varying sample size.}
\end{figure}
\end{document}